\documentclass[12pt,journal,letterpaper,draftcls,onecolumn]{IEEEtran}
\usepackage{graphicx,epsfig,amsmath,amsthm,amssymb,amsfonts,enumerate,bm,multirow,cite,subfigure}

\theoremstyle{plain}

\newtheorem{lem}{Lemma}
\newtheorem{cor}{Corollary}
\newtheorem{thm}{Theorem}

\begin{document}
\title{Generalized Ambiguity Decomposition for Understanding Ensemble Diversity}
\author{Kartik Audhkhasi$^1$, Abhinav Sethy$^2$,\\	Bhuvana Ramabhadran$^2$, Shrikanth S. Narayanan$^1$\\ \vspace{75pt}
$^1$Signal Analysis and Interpretation Lab (SAIL)\\Electrical Engineering Department\\University of Southern California, Los Angeles, USA\\Email: audhkhas@usc.edu, shri@sipi.usc.edu\\ \vspace{75pt}
$^2$IBM T. J. Watson Research Center\\ Yorktown Heights, New York, USA\\Email: \{asethy, bhuvana\}@us.ibm.com}

\maketitle
\newpage
\begin{abstract}
Diversity or complementarity of experts in ensemble pattern recognition and information processing systems is widely-observed by researchers to be crucial for achieving performance improvement upon fusion. Understanding this link between
ensemble diversity and fusion performance is thus an important research question. However, prior works
have theoretically characterized ensemble diversity and have linked it with ensemble performance in very restricted settings.
We present a generalized ambiguity decomposition (GAD) theorem as a broad framework for answering these questions. The GAD theorem applies to a generic convex ensemble of experts for any arbitrary twice-differentiable loss function. It shows that the ensemble performance approximately decomposes into
a difference of the average expert performance and the diversity of the ensemble. It thus provides a theoretical explanation for the empirically-observed benefit of fusing outputs from diverse classifiers and regressors. It also provides a loss function-dependent, ensemble-dependent, and data-dependent definition of diversity. We present extensions of this decomposition to common regression and classification loss functions, and report a simulation-based analysis of the diversity term and the accuracy of the decomposition. We finally present experiments on standard pattern recognition data sets which indicate the accuracy of the decomposition for real-world classification and regression problems.

\end{abstract}

\section{Introduction}
Researchers across several fields have empirically observed that an ensemble of multiple experts (classifiers or regressors) performs better than a single expert. Well-known examples demonstrating this performance benefit span a large variety of applications such as
\begin{itemize}
	\item \emph{Automatic speech and language processing}: Most teams in large-scale projects involving automatic speech recognition and processing such as the DARPA GALE~\cite{ibmgale2010}, CALO~\cite{tur2010calo}, RATS~\cite{mangu2013exploiting} and the IARPA BABEL~\cite{cui2013developing} programs use a combination of multiple systems for achieving state-of-the-art performance. The use of multiple systems is also widespread in text and natural language processing applications, with examples ranging from parsing~\cite{sagae2007dependency} to text categorization~\cite{sebastiani2002machine}.
	\item \emph{Recommendation systems}: Many industrial and academic teams competing for the Netflix prize~\cite{netflix2007} used ensembles of diverse systems for the movie rating prediction task. The \$1 Million grand prize winning system from the team \emph{BellKor's Pragmatic Chaos} was composed of multiple systems from three independent teams (\emph{BellKor~\cite{bell2007bellkor}, Pragmatic Theory~\cite{piotte2009pragmatic}, } and \emph{BigChaos~\cite{toscher2009bigchaos}}).
	\item \emph{Web information retrieval}: Researchers have also used ensembles of diverse systems for information retrieval tasks from the web. For instance, the winning teams~\cite{burges2011learning, pavlov2010bagboo} in the Yahoo! Learning to Rank Challenge~\cite{yahoo_ltr} used ensemble methods (such as bagging, boosting, random forests, and lambda-gradient models) for improving document ranking performance. The challenge overview article~\cite{yahoo_ltr} also emphasizes the performance benefits obtained by all teams when using ensemble methods.
	\item \emph{Computer vision}: Ensemble methods are often popular in computer vision tasks as well, such as tracking~\cite{avidan2007ensemble}, object detection~\cite{malisiewicz2011ensemble}, and pose estimation~\cite{gutta2000mixture}.
	\item \emph{Human state pattern recognition}: Systems for multimodal physical activity detection~\cite{li2010multimodal} often fuse classifiers trained on different feature sets for achieving an improvement in accuracy. Several teams competing in the Interspeech challenges have also used ensembles for classification of human emotion~\cite{schuller2009interspeech}, age and gender~\cite{schuller2010interspeech}, intoxication and sleepiness~\cite{schuller2011interspeech}, personality, likability, and pathology~\cite{schuller2012interspeech}, and social signals~\cite{schuller2013interspeech}.
\end{itemize}


 The above list is only a small fraction of the large number of applications which have used ensembles of multiple systems. Dietterich~\cite{dietterich2000} offers three main reasons for the observed benefits of an ensemble. First, an ensemble can potentially have a lower generalization error than a single expert. Second, the parameter estimation involved in training most state of the art expert systems such as neural networks involves solving a non-convex optimization problem. A single expert can get stuck in local optima whereas an ensemble of multiple experts can provide parameter estimates closer to the global optima. Finally, the true underlying function for a problem at hand may be too complex for a single expert and an ensemble may be better able to approximate it.


Intuition and documented research such as the ones listed above suggests that the experts in the ensemble should be optimally diverse. Diversity acts as a hedge against uncertainty in the evaluation data set, and the mismatch between the loss functions used for training and evaluation. Kuncheva~\cite{kuncheva2004book} gives a simple intuitive argument in favor of having the right amount of diversity in an ensemble. She says that just one expert suffices if all experts produce identical output, however, if the experts disagree in their outputs very frequently, it indicates that they are individually poor estimators of the target variable. Ambiguity decomposition~\cite{krogh1995} (AD) explains this tradeoff for the special case of the squared error loss function. Let $X \in \mathcal{X} \subseteq \mathcal{R}^D$ and $Y \in \mathcal{Y} \subseteq \mathcal{R}$ denote the $D$-dimensional input and $1$-dimensional target (output) random variables respectively. Let $f_k:\mathcal{X}\rightarrow\mathcal{R}$ be the $k^{th}$ expert which maps the input space $\mathcal{X}$ to the real line $\mathcal{R}$. $f$ is a convex combination of $K$ experts when
\begin{align}
 f(X) &= \sum_{k=1}^K w_k f_k(X) \quad \text{where } w_k \geq 0 \text{ and } \sum_{k=1}^K w_k = 1 \;.
\end{align}
AD states that the squared error between the above $f(X)$ and $Y$ is
\begin{align}\label{eq:amb_decomp}
 [Y-f(X)]^2 &= \sum_{k=1}^K w_k [Y-f_{k}(X)]^2 - \sum_{k=1}^K w_k [f_{k}(X)-f(X)]^2 \;.
\end{align}
The first term on the right hand side is the weighted squared error of the individual experts with respect to $Y$. The second term quantifies the diversity of the ensemble and is the squared error spread of the experts around $f(X)$. For two ensembles with identical weighted squared error, one with a greater diversity will have a lower overall squared error. The bias-variance-covariance decomposition is an equivalent result by Ueda and Nakano~\cite{ueda1996generalization} with neural network ensembles as the focus. We note that an ensemble of neural networks often consists of almost equally-accurate but diverse networks due to the non-convex training optimization problem. AD is also related to the bias-variance decomposition (BVD)~\cite{geman1992} which says that the expected squared error between a regressor $f_{\mathcal{D}}(X)$ trained on dataset $\mathcal{D}$ and the target variable $Y$ is
\begin{align}\label{eq:bvd_mse}
 \mathbb{E}_{\mathcal{D}}\{[f_{\mathcal{D}}(X) - Y]^2\} &= [Y - \mathbb{E}_{\mathcal{D}}\{f_{\mathcal{D}}(X)\}]^2 + \mathbb{E}_{\mathcal{D}}\{[f_{\mathcal{D}}(X) - \mathbb{E}_{\mathcal{D}}\{f_{\mathcal{D}}(X)\}]^2\} \;.
\end{align}
The first term on the right hand side is the square of the bias, which is the difference between the target $Y$ and the expected prediction over the distribution of $\mathcal{D}$. The second term measures the variance of the ensemble.
BVD reduces to AD when experts have the same functional form (e.g., linear) and when the training set $\mathcal{D}$ is drawn from a convex mixture of training sets $\{\mathcal{D}_k\}_{k=1}^K$ with mixture weights $\{w_k\}_{k=1}^K$.

Many existing algorithms attempt to promote diversity while training an ensemble. Examples include ensembles of decision trees~\cite{breiman2001random}, support vector machines~\cite{valentini2004bias}, conditional maximum entropy models~\cite{audhkhasi2012creating}, negative correlation learning~\cite{liu1999ensemble} for neural networks and DECORATE~\cite{melville2003constructing}, which is a meta-algorithm based on generation of synthetic data. AdaBoost~\cite{freund1995decision} is another prominent algorithm which incrementally creates a diverse ensemble of weak experts by modifying the distribution from which training instances are sampled. However, only few studies have focused on understanding the impact of diversity on ensemble performance for both classifiers and regressors. AD provides this link only for least squares regression. The analysis presented by Tumer and Ghosh~\cite{tumer1996analysis} assumes classification as regression over class posterior distribution. 

This paper presents a generalized ambiguity decomposition (GAD) theorem that is applicable to both classification and regression. It does not assume that the classifier is estimating a posterior distribution over the label set $\mathcal{Y}$ in case of classification. This is often encountered in practice, for example in case of support vector machines.
We note that some prior work has been done for deriving a BVD for a single expert with different loss functions~\cite{kong1995error,breiman1996arcing,domingos2000unified}. The proposed GAD theorem is different. It focuses on a convex combination of experts rather than a single expert. Even though one can link the BVD to AD by considering a mixture of training sets as mentioned before, this link requires that the individual experts should have the same functional form. We do not make such assumptions. Our result applies pointwise for any given $(X,Y) \in \mathcal{X}\times\mathcal{Y}$ rather than relying on an ensemble average.

We present the GAD theorem and its proof in the next section. We derive the decomposition for some common regression and classification loss functions in Section~\ref{sec:comm_loss}. We present a simulation-based analysis in Section~\ref{sec:sim}. We then evaluate the presented decomposition on multiple standard classification and regression data sets in Section~\ref{sec:expts}. Section~\ref{sec:concl} presents the conclusion and some directions for future work.

\section{Generalized Ambiguity Decomposition (GAD) Theorem}\label{sec:gad}
The concept of a loss function is central to statistical learning theory~\cite{vapnik1999nature}. It computes the mismatch between the prediction of an expert and the true target value. Lemma~\ref{lem:loss_bound} below presents useful bounds on a class of loss functions which are used widely in supervised machine learning.

\begin{lem}{\bf (Taylor's Theorem for Loss Functions~\cite{korn})}\label{lem:loss_bound}
Let $x, Y \in \mathcal{R}$ and $\mathcal{B} \subseteq \mathcal{R}$ be a closed and bounded set containing $x$. Let $l:\mathcal{R}\times\mathcal{R}\rightarrow\mathcal{R}$ be a loss function which is twice-differentiable in its second argument with continuous second derivative over $\mathcal{B}$. Let
\begin{align}
 M_{l,\mathcal{B}}(Y) &= \sup_{z \in \mathcal{B}} l''(Y,z) < \infty \quad\text{and}\\
 m_{l,\mathcal{B}}(Y) &= \inf_{z \in \mathcal{B}} l''(Y,z) > -\infty \;.
\end{align}
Then for any $Y_0 \in \mathcal{B}$, we can write the following quadratic bounds on the loss function:
\begin{align}
 l(Y,Y_0) &\geq l(Y,x) + l'(Y,x) (Y_0-x) + \frac{m_{l,\mathcal{B}}(Y)}{2} (Y_0-x)^2 \quad\text{and} \\
 l(Y,Y_0) &\leq l(Y,x) + l'(Y,x) (Y_0-x) + \frac{M_{l,\mathcal{B}}(Y)}{2} (Y_0-x)^2 \;.
\end{align}
\end{lem}
\begin{proof}
 Since $l(Y,Y_0)$ is twice-differentiable in its second argument over $\mathcal{R}\times\mathcal{B}$, by Taylor's theorem~\cite{korn}, $\exists$ a function $h_2:\mathcal{R}\times\mathcal{B}\rightarrow\mathcal{R}$ such that
\begin{align}\label{eq:taylor_thm}
 l(Y,Y_0) &= l(Y,x) + l'(Y,x) (Y_0-x) + h_2(Y,Y_0)(Y_0-x)^2 
\quad\text{where } \lim_{Y_0\rightarrow x} h_2(Y,Y_0) = 0
\end{align}
 for any given $x \in \mathcal{B}$. $h_2(Y,Y_0)(Y_0-x)^2$ is called remainder or residue and has the following form due to the Mean Value Theorem~\cite{korn}:
\begin{align}
 h_2(Y,Y_0)(Y_0-x)^2 &= \frac{l''(Y,z)}{2} (Y_0-x)^2 
\quad\text{where } z \in (Y_0,x) \;.
\end{align}
The second derivative of the loss function is continuous over the closed and bounded set $\mathcal{B}$. Weierstrass' Extreme Value Theorem~\cite{korn} gives:
\begin{align}
 m_{l,\mathcal{B}}(Y) &\leq l''(Y,z) \leq M_{l,\mathcal{B}}(Y) \quad \forall z \in \mathcal{B} \label{eq:2der_bounds} \\
 \text{where } & m_{l,\mathcal{B}}(Y) = \inf_{z \in \mathcal{B}} l''(Y,z) > -\infty\\
 \text{and } & M_{l,\mathcal{B}}(Y) = \sup_{z \in \mathcal{B}} l''(Y,z) < \infty \;.
\end{align}
We note that $m_{l,\mathcal{B}}(Y) = 0$ is an obvious choice if $l$ is convex. Using the bounds in (\ref{eq:2der_bounds}) in Taylor's theorem from (\ref{eq:taylor_thm}) results in the desired inequalities:
\begin{align}
 l(Y,Y_0) &\geq l(Y,x) + l'(Y,x) (Y_0-x) + \frac{m_{l,\mathcal{B}}(Y)}{2} (Y_0-x)^2 \\
 l(Y,Y_0) &\leq l(Y,x) + l'(Y,x) (Y_0-x) + \frac{M_{l,\mathcal{B}}(Y)}{2} (Y_0-x)^2 \;.
\end{align}
\end{proof}
The second argument of $l$ is always bounded in practice since it represents the prediction of the expert. Hence, limiting the domain of twice-differentiability and continuity of the second derivative from $\mathcal{R}\times\mathcal{R}$ to $\mathcal{R}\times\mathcal{B}$ is a reasonable assumption. The next lemma presents ambiguity decomposition for the squared error loss function~\cite{krogh1995}. We denote $f(X)$ as $f$ and $f_k(X)$ as $f_k$ from now on for notational simplicity.

\begin{lem}\label{lem:ad} {\bf (Ambiguity Decomposition (AD)~\cite{krogh1995})}\\
Consider an ensemble of $K$ experts $\{f_k:\mathcal{X}\rightarrow\mathcal{R}, k = 1,2,\ldots,K\}$ and let $f = \sum_{k=1}^K w_k f_k$ be a convex combination of these experts. Then
\begin{align}
 [Y-f]^2 &= \sum_{k=1}^K w_k [Y-f_{k}]^2 - \sum_{k=1}^K w_k [f_k-f]^2 \quad\forall (X,Y) \in \mathcal{X}\times\mathcal{R} \;.
\end{align}
\end{lem}
\begin{proof}
 We start by expanding the following term:
 \begin{align}
  \sum_{k=1}^K w_k [Y-f_{k}]^2 &= \sum_{k=1}^K w_k [Y-f - (f_k-f)]^2 \\
  = \sum_{k=1}^K w_k [Y-f]^2 + \sum_{k=1}^K & w_k [f_k - f]^2 - 2\sum_{k=1}^K w_k [Y-f][f_k-f] \\
  = [Y-f]^2 + \sum_{k=1}^K w_k & [f_k - f]^2 - 2[Y-f]\sum_{k=1}^K w_k [f_k - f] \\
  = [Y-f]^2 + \sum_{k=1}^K w_k & [f_k - f]^2 - 2[Y-f] [\sum_{k=1}^K w_k f_k - f] \\
  = [Y-f]^2 + \sum_{k=1}^K w_k & [f_k - f]^2 \;.
 \end{align}
 We arrive at the Ambiguity Decomposition by re-arranging terms in the above equation.
\begin{align}
 [Y-f]^2 &= \sum_{k=1}^K w_k [Y-f_{k}]^2 - \sum_{k=1}^K w_k [f_k-f]^2 \;.
\end{align}
\end{proof}

Ambiguity decomposition describes the tradeoff between the accuracy of individual experts and the diversity of the ensemble. But it applies only to the squared error loss function. We now state and prove the Generalized Ambiguity Decomposition (GAD) theorem using Lemmas~\ref{lem:loss_bound} and \ref{lem:ad}.

\begin{thm}\label{thm:gad}{\bf (Generalized Ambiguity Decomposition (GAD) Theorem)}\\
Consider an ensemble of $K$ experts $\{f_k:\mathcal{X}\rightarrow\mathcal{R}, k = 1,2,\ldots,K\}$ and let $f = \sum_{k=1}^K w_k f_k$ be a convex combination of these experts. Assume that all $f_k$ are finite. Let $(X,Y) \in \mathcal{X}\times\mathcal{R}$ and let $\mathcal{B} \subseteq \mathcal{R}$ be the following closed and bounded set:
\begin{align}
 \mathcal{B} &= [b_\text{min},b_\text{max}] \quad\text{where} \\
 b_\text{min} &= \min\{Y,f_1,\ldots,f_K\} \quad\text{and} \\
 b_\text{max} &= \max\{Y,f_1,\ldots,f_K\} \;.
\end{align}
$\mathcal{B}$ is the smallest closed and bounded set which contains $Y$ and all $f_k$. Let $l:\mathcal{R}\times\mathcal{R}\rightarrow\mathcal{R}$ be a loss function which is twice-differentiable in its second argument with continuous second derivative over $\mathcal{B}$. Let:
\begin{align}
 M_{l,\mathcal{B}}(Y) &= \sup_{z \in \mathcal{B}} l''(Y,z) < \infty \;, \\
 M_{l,\mathcal{B}}(f) &= \sup_{z \in \mathcal{B}} l''(f,z) \in (0,\infty) \;, \quad\text{and} \\
 m_{l,\mathcal{B}}(Y) &= \inf_{z \in \mathcal{B}} l''(Y,z) > -\infty \;.
\end{align}
Then the ensemble loss is upper-bounded as given below:
\begin{align}
 l(Y,f) &\leq \sum_{k=1}^K w_k l(Y,f_k) - \frac{M_{l,\mathcal{B}}(Y)}{M_{l,\mathcal{B}}(f)} \Big{[}\sum_{k=1}^K w_k l(f,f_k) - l(f,f)\Big{]} \nonumber \\
 &+ \frac{1}{2}\Big{(}M_{l,\mathcal{B}}(Y) - m_{l,\mathcal{B}}(Y)\Big{)}\sum_{k=1}^K w_k (Y-f_k)^2 \;.
\end{align}
\end{thm}
\begin{proof}
 $\mathcal{B}$ is a closed and bounded set which includes $Y$ and all $f_k$ by definition. Hence we can write the following lower-bound for $l(Y,f_k)$ using Lemma~\ref{lem:loss_bound}:
\begin{align}
 l(Y,f_k) &\geq l(Y,Y) + l'(Y,Y)(f_k - Y) + \frac{m_{l,\mathcal{B}}(Y)}{2}(f_k-Y)^2 \;.
\end{align}
Taking a convex sum on both sides of the above inequality gives
\begin{align}\label{eq:gad_proof_bound1}
 \sum_{k=1}^K w_k l(Y,f_k) &\geq \sum_{k=1}^K w_k l(Y,Y) + \sum_{k=1}^K w_k l'(Y,Y)(f_k - Y) + \sum_{k=1}^K w_k\frac{m_{l,\mathcal{B}}(Y)}{2}(f_k-Y)^2 \nonumber \\
&= l(Y,Y) + l'(Y,Y) (f - Y) + \frac{m_{l,\mathcal{B}}(Y)}{2}\sum_{k=1}^K w_k(f_k-Y)^2 \;.
\end{align}
$\mathcal{B}$ also includes $f$ because it includes all $f_k$ and $f$ is their convex combination. Thus, we consider the following upper-bound on $l(Y,f)$ using Lemma~\ref{lem:loss_bound}:
\begin{align}
 l(Y,f) &\leq l(Y,Y) + l'(Y,Y)(f-Y) + \frac{M_{l,\mathcal{B}}(Y)}{2}(f-Y)^2 \\
\iff l(Y,Y) &+ l'(Y,Y)(f-Y) \geq l(Y,f) - \frac{M_{l,\mathcal{B}}(Y)}{2}(f-Y)^2 \;.
\end{align}
Substituting this inequality in (\ref{eq:gad_proof_bound1}) gives
\begin{align}
 \sum_{k=1}^K w_k l(Y,f_k) &\geq l(Y,f) - \frac{M_{l,\mathcal{B}}(Y)}{2}(f-Y)^2 + \frac{m_{l,\mathcal{B}}(Y)}{2}\sum_{k=1}^K w_k(f_k-Y)^2 \;.
\end{align}
We use AD in Lemma~\ref{lem:ad} for $(f-Y)^2$ and write the above bound as:
\begin{align}\label{eq:gad_proof_bound2}
 \sum_{k=1}^K w_k l(Y,f_k) &\geq l(Y,f) - \frac{1}{2}(M_{l,\mathcal{B}}(Y)-m_{l,\mathcal{B}}(Y))\sum_{k=1}^K w_k (f_k - Y)^2 
+ \frac{M_{l,\mathcal{B}}(Y)}{2}\sum_{k=1}^K w_k (f_k - f)^2 \;.
\end{align}
We finally invoke the following upper bound on $l(f,f_k)$ using Lemma~\ref{lem:loss_bound}:
\begin{align}
 l(f,f_k) \leq l(f,f) &+ l'(f,f)(f-f_k) + \frac{M_{l,\mathcal{B}}(f)}{2} (f - f_k)^2 \\
\iff \frac{M_{l,\mathcal{B}}(f)}{2} (f - f_k)^2 &\geq l(f,f_k) - l(f,f) - l'(f,f)(f-f_k) \\
\iff \frac{M_{l,\mathcal{B}}(f)}{2} \sum_{k=1}^K w_k(f - f_k)^2 &\geq \sum_{k=1}^K w_k l(f,f_k) - l(f,f) \;.
\end{align}
We get the desired result by substituting the above inequality in (\ref{eq:gad_proof_bound2}) and using the fact that $M_l(f) > 0$:
\begin{align}
 \sum_{k=1}^K w_k l(Y,f_k) &\geq l(Y,f) - \frac{1}{2}(M_{l,\mathcal{B}}(Y)-m_{l,\mathcal{B}}(Y))\sum_{k=1}^K w_k (f_k - Y)^2 \nonumber \\
 + \frac{M_{l,\mathcal{B}}(Y)}{M_{l,\mathcal{B}}(f)} \Big{[}& \sum_{k=1}^K w_k l(f,f_k) - l(f,f) \Big{]} \\
  \iff l(Y,f) &\leq \sum_{k=1}^K w_k l(Y,f_k) - \frac{M_{l,\mathcal{B}}(Y)}{M_{l,\mathcal{B}}(f)} \Big{[}\sum_{k=1}^K w_k l(f,f_k) - l(f,f)\Big{]} \nonumber \\
 &+ \frac{1}{2}(M_{l,\mathcal{B}}(Y) - m_{l,\mathcal{B}}(Y))\sum_{k=1}^K w_k (Y-f_k)^2 \;.
\end{align}
\end{proof}
The GAD Theorem is a natural extension of AD in Lemma~\ref{lem:ad} and reduces to it for the case of squared error loss. We can gain more intuition about this result by defining the following quantities:
\begin{align}
 \text{Ensemble loss:} &\quad l(Y,f) \\
 \text{Weighted expert loss:} &\quad \sum_{k=1}^K w_k l(Y,f_k) \\
 \text{Diversity:} &\quad d_l(f_1,\ldots,f_K) = \frac{M_{l,\mathcal{B}}(Y)}{M_{l,\mathcal{B}}(f)} \Bigg{[} \sum_{k=1}^K w_k l(f,f_k) - l(f,f)\Bigg{]} \\
 \text{Curvature spread (CS):} &\quad s_{l,\mathcal{B}}(Y) = M_{l,\mathcal{B}}(Y) - m_{l,\mathcal{B}}(Y) \geq 0
\end{align}
 Ignoring the term involving curvature spread, GAD says that the ensemble loss is upper-bounded by weighted expert loss minus the diversity of the ensemble. Thus, the upper-bound involves a tradeoff between the performance of individual experts (weighted experts loss) and the diversity. Diversity measures the spread of the expert predictions about the ensemble's predictions and is $0$ when $f_k = f, \forall k$. Diversity is non-negative for a convex loss function due to Jensen's inequality~\cite{korn}. Furthermore, diversity depends on the loss function, the true target $Y$ and the prediction of the ensemble $f$ at the current data point. Thus, all data points are not equally important from a diversity perspective. It is also interesting to note that the GAD theorem provides a decomposition of the ensemble loss into a supervised (weighted expert loss) and unsupervised (diversity) term. The latter term does not require labeled data to compute. This makes the overall framework applicable to semi-supervised settings.

The following corollary of Theorem~\ref{thm:gad} gives a simple upper-bound on the error between $l(Y,f)$ and its approximation motivated by the GAD theorem.

\begin{cor}\label{cor:gad}{\bf (Error Bound for GAD Loss Function Approximation)}\\
 If
\begin{align}
 \sum_{k = 1}^{K}w_k(Y - f_k)^2 &= \beta(X,Y) \quad\text{and} \\
 \max_{k\in\{1,\ldots,K\}}(Y - f_k)^2 &= \delta(X,Y) \;,
\end{align}
then the error between the true loss and its GAD approximation is bounded as:
\begin{align}
 l(Y,f) - l_{\text{GAD}}(Y,f) &\leq \frac{1}{2} s_{l,\mathcal{B}}(Y,f)\beta(X,Y) \\
 &\leq \frac{1}{2} s_{l,\mathcal{B}}(Y,f)\delta(X,Y) \;,
\end{align}
where $s_{l,\mathcal{B}}(Y,f)$ is the curvature spread defined previously and
\begin{align}
l_{\text{GAD}}(Y,f) &= \sum_{k=1}^K w_k l(Y,f_k) - d_l(f_1,\ldots,f_K)
\end{align}
is an approximation for $l(Y,f)$ motivated by GAD.
\end{cor}
\begin{proof}
 Theorem~\ref{thm:gad} gives:
\begin{align}
 l(Y,f) - l_{\text{GAD}}(Y,f) &\leq \frac{1}{2} (M_{l,\mathcal{B}}(Y) - m_{l,\mathcal{B}}(Y))\sum_{k=1}^K w_k (Y-f_k)^2 \nonumber \\
 &= \frac{1}{2}(M_{l,\mathcal{B}}(Y) - m_{l,\mathcal{B}}(Y))\beta(X,Y) \;.
\end{align}
We also note that:
\begin{align}
	\sum_{k=1}^K w_k (Y-f_k)^2 &\leq \max_{k\in\{1,\ldots,K\}}(Y - f_k)^2 = \delta(X,Y) \;.
\end{align}
Hence we can also write the following less tight upper bound on the error:
\begin{align}
 l(Y,f) - l_{\text{GAD}}(Y,f) &\leq \frac{1}{2} (M_{l,\mathcal{B}}(Y) - m_{l,\mathcal{B}}(Y))\delta(X,Y) \;.
\end{align}
\end{proof}
Corollary~\ref{cor:gad} shows that $l_{\text{GAD}}(Y,f)$ is a good approximation for $l(Y,f)$ when the curvature spread is small and all expert predictions are close to the true target $Y$. For instances $(X,Y)$ where multiple experts in the ensemble are far away from the true target, $l_{\text{GAD}}(Y,f)$ has a high error. To summarize, the accuracy of $l_{\text{GAD}}$ depends on the data instance, loss function and the expert predictions.

We note that the diversity term in the GAD theorem computes the loss function between each expert $f_k$ and the ensemble prediction $f$. However, it is sometimes useful to understand diversity in terms of pairwise loss functions between the expert predictions themselves. The next corollary to the GAD theorem shows that we can indeed re-write the diversity term in pairwise fashion for a metric loss function.

\begin{cor}\label{cor:pairwise_gad}{\bf (Pairwise GAD Theorem for Metric Loss Functions)}
	Consider a metric loss function $l$ and also let $w_k = 1/K~\forall k$ for simplicity. Then the GAD theorem becomes
	\begin{align}
		l(Y,f) &\leq \frac{1}{K}\sum_{k=1}^K l(Y,f_k) - \frac{M_{l,\mathcal{B}}(Y)}{M_{l,\mathcal{B}}(f)} \Bigg{[}\frac{1}{K(K-1)}\sum_{k_1=1}^K\sum_{k_2=k_1+1}^K l(f_{k_1},f_{k_2}) \Bigg{]} \nonumber \\
 &+ \frac{1}{2}\Big{(}M_{l,\mathcal{B}}(Y) - m_{l,\mathcal{B}}(Y)\Big{)}\sum_{k=1}^K (Y-f_k)^2 \;.
	\end{align}
\end{cor}
\begin{proof}
	The loss function satisfies the triangle inequality because it is given to be a metric. We visualize the output of each expert and the ensemble's prediction as points in a metric space induced by the metric loss function. Hence
	\begin{align}
		l(f_{k_1},f_{k_2}) &\leq l(f,f_{k_1}) + l(f,f_{k_2})
	\end{align}
	for all $k_1 \in \{1,\ldots,K\}$ and $k_2 \in \{k_1+1,\ldots,K\}$. We add these $K(K-1)$ inequalities to get
	\begin{align}
		\sum_{k_1=1}^K\sum_{k_2=k_1+1}^K l(f_{k_1},f_{k_2}) &\leq (K-1)\sum_{k=1}^K l(f,f_k) \;.
	\end{align}
	We also note that $l(f,f) = 0$ because $l$ is a metric. Hence we get the following lower bound on the diversity term in GAD
	\begin{align}
		\frac{M_{l,\mathcal{B}}(Y)}{M_{l,\mathcal{B}}(f)} \Bigg{[}\frac{1}{K}\sum_{k=1}^K l(f,f_k) - l(f,f) \Bigg{]} &\geq \frac{M_{l,\mathcal{B}}(Y)}{M_{l,\mathcal{B}}(f)} \Bigg{[}\frac{1}{K(K-1)}\sum_{k_1=1}^K\sum_{k_2=k_1+1}^K l(f_{k_1},f_{k_2}) \Bigg{]}
	\end{align}
	Substituting this lower-bound in GAD from Theorem~\ref{thm:gad} gives the desired decomposition with pairwise diversity.
\end{proof}

The squared error and absolute error loss functions used for regression are metric functions and thus permit a decomposition with a pairwise diversity term as given in Corollary~\ref{cor:pairwise_gad} above. We now derive the quantities required for GAD approximation of common loss functions in the next section.

\section{GAD for Common Loss Functions}\label{sec:comm_loss}
%

We note that the computation of $M_{l,\mathcal{B}}(Y)$ and $m_{l,\mathcal{B}}(Y)$ is critical to the GAD theorem. Hence the following subsections focus on deriving these quantities for various common classification and regression loss functions.

\subsection{Squared Error Loss}\label{subsec:sqr_err_loss}
Squared error is the most common loss function used for regression and is defined as given below:
\begin{align}
 l_{\text{sqr}}(Y,Y_0) &= (Y-Y_0)^2 \quad \text{where } Y, Y_0 \in \mathcal{R}\;.
\end{align}
Its second derivative with respect to $Y_0$ is $l''_{\text{sqr}}(Y,Y_0) = 2$. Hence $M_{l,\mathcal{B}}(Y) = m_{l,\mathcal{B}}(Y) = 2~\forall Y$ and CS is $0$. Thus GAD reduces to AD in Lemma~\ref{lem:ad}.

\subsection{Absolute Error Loss}\label{subsec:abs_err_loss}
Absolute error loss function is more robust than squared error for outliers and is defined as:
\begin{align}
 l_{\text{abs}}(Y,Y_0) &= |Y-Y_0| \quad \text{where } Y, Y_0 \in \mathcal{R} \;.
\end{align}
This function is not differentiable at $Y_0 = Y$. We thus consider two commonly used smooth approximations to the absolute error loss function. The first one uses the integral of the inverse tangent function which approximates the sign function. This leads to the following approximation:
\begin{align}
 l_{\text{abs, approx1}}(Y,Y_0) &= \frac{2(Y-Y_0)}{\pi} \tan^{-1}\Bigg{(}\frac{Y-Y_0}{\epsilon}\Bigg{)} \quad \text{where } Y, Y_0 \in \mathcal{R} \text{ and } \epsilon > 0 \;.
\end{align}
One can get an arbitarily close approximation by setting a suitably small positive value of $\epsilon$. The second derivative of the loss function with respect to $Y_0$ is
\begin{align}
 l''_{\text{abs, approx1}}(Y,Y_0) &= \frac{4}{\epsilon\pi\Big{[}1+\Big{(}\frac{Y-Y_0}{\epsilon}\Big{)}^2\Big{]}^2} \;.
\end{align}
We need to compute the maximum and minimum of the above second derivative for $Y_0 \in \mathcal{B}$. The above function is monotonically increasing for $Y_0 < Y$, achieves its maxima at $Y_0 = Y$, and monotonically decreases for $Y_0 \geq Y$. We note that $Y \in \mathcal{B}$ by definition of $\mathcal{B}$. Hence, the maximum of $l''_{\text{abs, approx1}}(Y,Y_0)$ over $\mathcal{B}$ occurs at $Y_0 = Y$ and is given by:
\begin{align}
M_{l,\mathcal{B}}(Y) &= l''_{\text{abs, approx1}}(Y,Y) = \frac{4}{\pi\epsilon}\;.
\end{align}
The minimum value depends on the location of $\mathcal{B} = [b_\text{min},b_\text{max}]$ and is given below:
\begin{align}
 m_{l,\mathcal{B}}(Y) &= \left\{
	\begin{array}{ll}
	l''_{\text{abs, approx1}}(Y,b_\text{min}) &;\text{ if } b_\text{max} + b_\text{min} < 2Y \\
	l''_{\text{abs, approx1}}(Y,b_\text{max}) &;\text{ otherwise}
	\end{array}\;.
 \right.
\end{align}
We also consider a second smooth approximation of absolute error:
\begin{align}
 l_{\text{abs, approx2}}(Y,Y_0) &= \sqrt{(Y-Y_0)^2 + \epsilon} - \sqrt{\epsilon} \quad \text{where } Y, Y_0 \in \mathcal{R} \text{ and } \epsilon > 0 \;.
\end{align}
This approximation becomes better with smaller positive values of $\epsilon$. The second derivative of the above approximation with respect to $Y_0$ is
\begin{align}
 l''_{\text{abs, approx2}}(Y,Y_0) &= \frac{\epsilon}{[(Y-Y_0)^2 + \epsilon]^{3/2}} \;.
\end{align}
The behavior of the above function with $Y_0$ is the same as $l''_{\text{abs, approx1}}(Y,Y_0)$. It has a monotonic increase for $Y_0 < Y$, achieves maxima at $Y_0 = Y$, and has a monotonic decrease for $Y_0 \geq Y$. This results in the following second derivative maxima and minima over $\mathcal{B}$:
\begin{align}
 M_{l,\mathcal{B}}(Y) &= l''_{\text{abs,approx2}}(Y,Y) = \frac{1}{\sqrt{\epsilon}} \;.
\end{align}
\begin{align}
 m_{l,\mathcal{B}}(Y) &= \left\{
	\begin{array}{ll}
	l''_{\text{abs, approx2}}(Y,b_\text{min}) &;\text{ if } b_\text{max} + b_\text{min} < 2Y \\
	l''_{\text{abs, approx2}}(Y,b_\text{max}) &;\text{ otherwise}
	\end{array}\;.
 \right. 
\end{align}
Both types of smooth absolute error loss functions give a non-zero curvature spread when compared to the squared error loss function. This leads to a non-zero approximation error for the GAD theorem.

\subsection{Logistic Loss}\label{subsec:log_loss}
Logistic regression is a popular technique for classification. We consider the binary classification case where the label set $\mathcal{Y}$ = \{-1,1\}. The logistic loss function is
\begin{align}
 l_{\text{log}}(Y,Y_0) &= \log(1+\exp(-YY_0)) \quad \text{where } Y \in \{-1,1\} \text{ and } Y_0 \in \mathcal{R} \;.
\end{align}
$Y_0$ is replaced by the expert's prediction for supervised learning and is typically modeled by an affine function of $X$. The ensemble is thus a convex combination of affine experts. The second derivative of the above loss with respect to $Y_0$ is
\begin{align}
 l''_{\text{log}}(Y,Y_0) &= \frac{Y^2 \exp(-YY_0)}{(1+\exp(-YY_0))^2}\;.
\end{align}
$l''_{\text{log}}(Y,Y_0)$ is an even function of $Y_0$. It is monotonically increasing for $Y_0 \leq 0$, reaches its maximum at $Y_0 = 0$, and is monotonically decreasing for $Y_0 \geq 0$. Hence we can write $m_{l,\mathcal{B}}(Y)$ as
\begin{align}
 m_{l,\mathcal{B}}(Y) &= \left\{
	\begin{array}{ll}
	 l''_{\text{log}}(Y,b_\text{min}) &;\text{ if } b_\text{max} < 0 \text{ or if } b_\text{max} + b_\text{min} < 0 \\
	 l''_{\text{log}}(Y,b_\text{max}) &;\text{ otherwise}
	\end{array}\;.
 \right.
\end{align}
Similarly, we can write $M_{l,\mathcal{B}}(Y)$ as
\begin{align}
  M_{l,\mathcal{B}}(Y) &= \left\{
	\begin{array}{ll}
	 l''_{\text{log}}(Y,b_\text{max}) &;\text{ if } b_\text{max} < 0 \\
	 l''_{\text{log}}(Y,b_\text{min}) &;\text{ if } b_\text{min} > 0 \\
	 l''_{\text{log}}(Y,0) = Y^2/4 &;\text{ otherwise}
	\end{array}\;.
 \right.
\end{align}


\subsection{Exponential Loss}\label{subsec:exp_loss}
AdaBoost.M1~\cite{freund1995decision} uses the exponential loss function which is defined as
\begin{align}
 l_{\text{exp}}(Y,Y_0) &= \exp(-YY_0) \quad \text{where } Y \in \{-1,1\} \text{ and } Y_0 \in \mathcal{R} \;.
\end{align}
The second derivative of the loss function is
\begin{align}
 l''_{\text{exp}}(Y,Y_0) &= Y^2 \exp(-YY_0) \;.
\end{align}
The above function of $Y_0$ is monotonically increasing when $Y < 0$ and monotonically decreasing when $Y \geq 0$. Hence $m_{l,\mathcal{B}}(Y)$ becomes
\begin{align}
 m_{l,\mathcal{B}}(Y) &= \left\{
	\begin{array}{ll}
	l''_{\text{exp}}(Y,b_\text{min}) &;\text{ if } Y < 0 \\ 
	l''_{\text{exp}}(Y,b_\text{max}) &;\text{ otherwise}
	\end{array} \;.
 \right.
\end{align}
Similarly, $M_{l,\mathcal{B}}$ becomes
\begin{align}
 M_{l,\mathcal{B}}(Y) &= \left\{
	\begin{array}{ll}
	l''_{\text{exp}}(Y,b_\text{max}) &;\text{ if } Y < 0 \\ 
	l''_{\text{exp}}(Y,b_\text{min}) &;\text{ otherwise}
	\end{array}\;.
 \right.
\end{align}


\subsection{Hinge Loss}\label{subsec:hinge_loss}
The hinge loss is another popular loss function which is used for training support vector machines (SVMs)~\cite{cortes1995support} and is defined as
\begin{align}
 l_{\text{hinge}}(Y,Y_0) &= \max(0,1-YY_0) \quad \text{where } Y \in \{-1,1\} \text{ and } Y_0 \in \mathcal{R} \;.
\end{align}
The above loss function is not differentiable when $YY_0 = 1$. Hence we use the following smooth approximation from Smooth SVM (SSVM)~\cite{lee2001ssvm}:
\begin{align}
 l_{\text{hinge, smooth}}(Y,Y_0) &= 1 - YY_0 + \epsilon\log\Bigg{[}1+\exp\Bigg{(}-\frac{1-YY_0}{\epsilon}\Bigg{)}\Bigg{]} \quad \text{where } \epsilon > 0 \;.
\end{align}
The above approximation is based on the logistic sigmoidal approximation of the sign function which is often used in neural networks~\cite{kosko1992neural}. Picking a small positive value of $\epsilon$ ensures low approximation error. The second derivative with respect to $Y_0$ is
\begin{align}
 l''_{\text{hinge, smooth}}(Y,Y_0) &= \frac{Y^2\exp\Big{(}-\frac{1-YY_0}{\epsilon}\Big{)}}{\epsilon\Big{[}1+\exp\Big{(}-\frac{1-YY_0}{\epsilon}\Big{)}\Big{]}^2} \;.
\end{align}
The above function of $Y_0$ is symmetrical about $Y_0=1/Y$, increases for $Y_0 < 1/Y$, attains its maximum for $Y_0 = 1/Y$, and decreases for $Y_0 \geq 1/Y$. Hence $M_{l,\mathcal{B}}(Y)$ is
\begin{align}
	M_{l,\mathcal{B}}(Y) &= \left\{
	\begin{array}{ll}
	l''_{\text{hinge, smooth}}(Y,b_\text{max}) &;\text{ if } b_\text{max} < 1/Y \\
	l''_{\text{hinge, smooth}}(Y,b_\text{min}) &;\text{ if } b_\text{min} > 1/Y \\
	l''_{\text{hinge, smooth}}(Y,1/Y) \text{ or } Y^2/(4\epsilon) &;\text{ otherwise}
	\end{array}\;.
 \right.
\end{align}

Similarly, the value of $m_{l,\mathcal{B}}(Y)$ also depends on the location of the interval $\mathcal{B}$ and is given below:
\begin{align}
 m_{l,\mathcal{B}}(Y) &= \left\{
	\begin{array}{ll}
	l''_{\text{hinge, smooth}}(Y,b_\text{min}) &;\text{ if } b_\text{max} < 1/Y \text{or if } b_\text{max} + b_\text{min} < 2/Y \\
	l''_{\text{hinge, smooth}}(Y,b_\text{max}) &;\text{ otherwise}
	\end{array}\;.
 \right.
\end{align}
The expressions for $M_{l,\mathcal{B}}(Y)$ and $m_{l,\mathcal{B}}(Y)$ derived in this section for various loss functions are used to derive the GAD approximation for
the ensemble loss. Theoretical analysis of this approximation is not easy for all loss functions. Hence the next section presents simulation experiments for understanding the GAD theorem.

\section{Simulation Experiments on the GAD Theorem for Common Loss Functions}\label{sec:sim}
This section begins by understanding the tradeoff between the diversity term and weighted expert loss in the GAD theorem. We next analyze the accuracy of the ensemble loss approximation motivated by the GAD theorem.
We finally contrast the GAD approximation with the Taylor series approximation used in gradient boosting.

\subsection{Behavior of Weighted Expert Loss and Diversity in GAD}
Consider the following proxy for the true loss function implied by GAD:
\begin{align}\label{eq:lgad}
	l_{\text{GAD}}(Y,f) &= \sum_{k=1}^K w_k l(Y,f_k) - d_l(f_1,\ldots,f_K) \;.
\end{align}
where $d_l(f_1,\ldots,f_K)$ is the diversity. The first term on the right hand side of the above equation is the weighted sum of the individual expert's losses. We note that this term provides a simple upper bound on $l(Y,f)$ due to Jensen's inequality for convex loss functions:
\begin{align}\label{eq:wgt_exp_loss}
	l(Y,f) &\leq \sum_{k=1}^K w_k l(Y,f_k) = l_\text{WGT}(Y,f) \;.
\end{align}

To understand the tradeoff between the two terms on the right hand side of Eq.~\ref{eq:lgad}, we performed Monte Carlo simulations because the analytical forms of $d_l(f_1,\ldots,f_K)$ for common loss functions derived in the previous section are not amenable to direct theoretical analysis. The $K$ expert predictions were sampled from an independent and identically distributed (IID) Gaussian random variable with mean $\mu_f$ and variance $\sigma_f^2$. That is
\begin{align}
	f_k &\sim \mathcal{N}(\mu_f,\sigma_f^2) \; \text{ for all } k \in \{1,\ldots,K\} \;.
\end{align}

A unimodal distribution was used since it is intuitive to expect most of the experts to give numerically close predictions. We used a Gaussian probability density function (PDF) for our simulations since it is the most popular unimodal PDF. The convex nature of the ensemble ensures that $\mu_f$ is also the expected ensemble prediction. This is because
\begin{align}
	\mathbb{E}\{f\} &= \mathbb{E}\Big{\{}\sum_{k=1}^K w_k f_k\Big{\}} = \sum_{k=1}^K w_k  \mathbb{E}\{f_k\} = \sum_{k=1}^K w_k  \mu_f = \mu_f \;.
\end{align}

The variance $\sigma_f^2$ governs the spread of the predictions around the mean. We varied $\mu_f$ around the true label $Y=1$. We picked $Y=1$ because the two regression loss functions depend only on the distance of the prediction from the target. The analysis also extends easily to $Y=-1$ for classification loss functions. We generated $1000$ Monte Carlo samples for $K=3$ and $7$ experts. We set $\sigma_f^2 = 2$ for these simulations. Figures~\ref{fig:sqr_loss_decomp}-\ref{fig:hin_loss_decomp} show the median values of $l(Y,f)$, $l_\text{GAD}(Y,f)$, and the weighted expert loss for various loss functions. These figures also plot the median diversity term $d_l(f_1,\ldots,f_K)$ with the expected ensemble prediction $\mu_f$.

We first analyse the plots for the two regression loss functions. Figure~\ref{fig:sqr_loss_decomp} shows the case for the squared error loss function. We note that $l(Y,f)$ and $l_\text{GAD}(Y,f)$ overlap for all values of $\mu_f$ because the GAD theorem reduces to the ambiguity decomposition. The diversity term also remains nearly constant because it is the maximum likelihood estimator of the variance $\sigma_f^2$. Diversity corrects for the bias between the weighted expert loss (green curve) and the actual ensemble loss function (black curve). Figure~\ref{fig:abs_loss_decomp} shows the corresponding figure for the smooth absolute error loss function with $\epsilon = 0.5$. $l_\text{GAD}(Y,f)$ provides a very accurate approximation of the true loss function around the true label $Y=1$ and becomes a poorer approximation as we move away. This is because GAD assumes the experts predictions to be close to the true label. $l_\text{WGT}(Y,f)$ gives a much larger approximation error in comparison around $Y=1$. Also, the diversity term is nearly constant because its computation normalizes for the value of $\mu_f$ by subtracting $f$ from each $f_k$.

\begin{figure}[h]
 \centering
 \mbox{
 \subfigure{\includegraphics[width=0.5\textwidth]{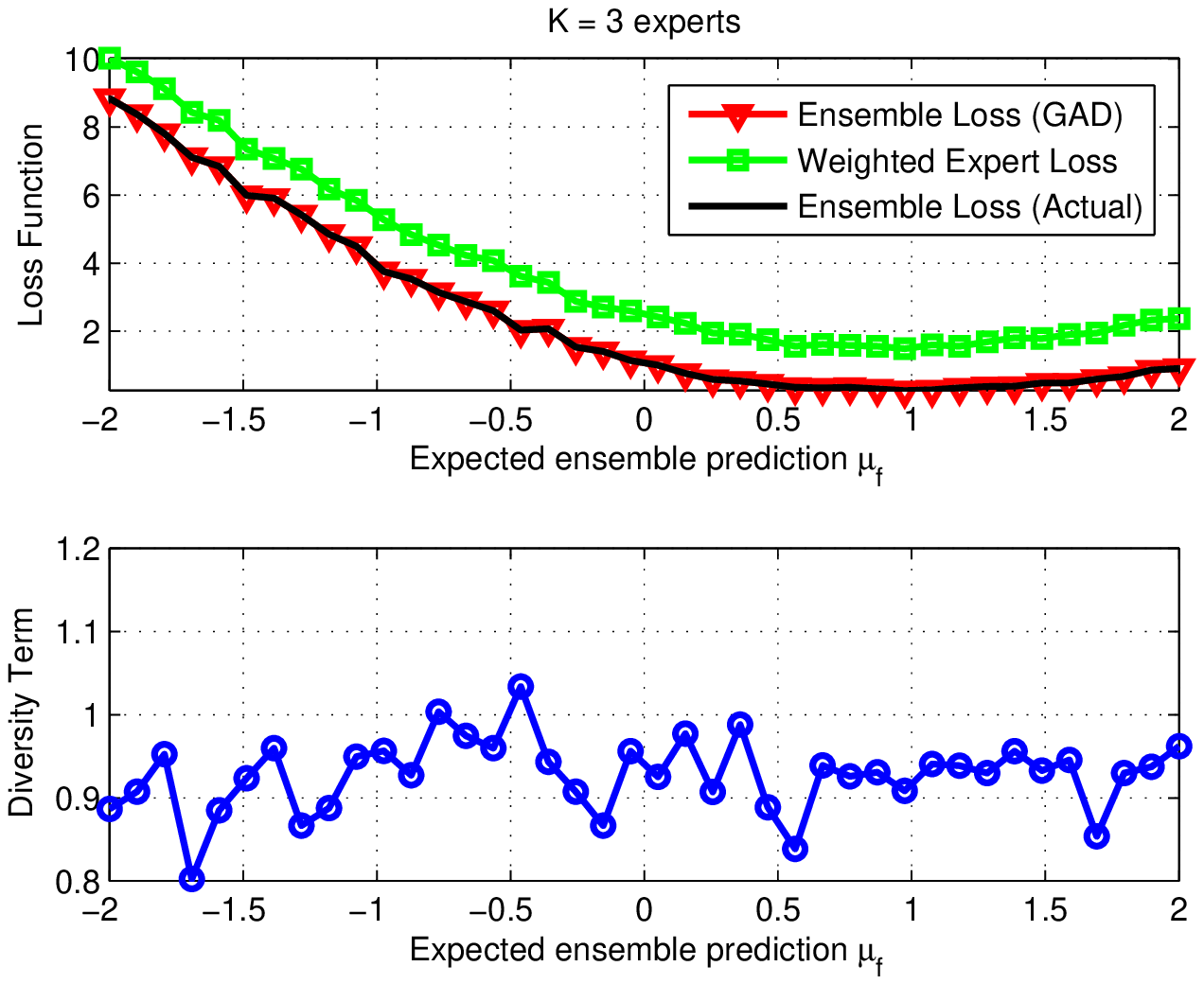}}
 \quad
 \subfigure{\includegraphics[width=0.5\textwidth]{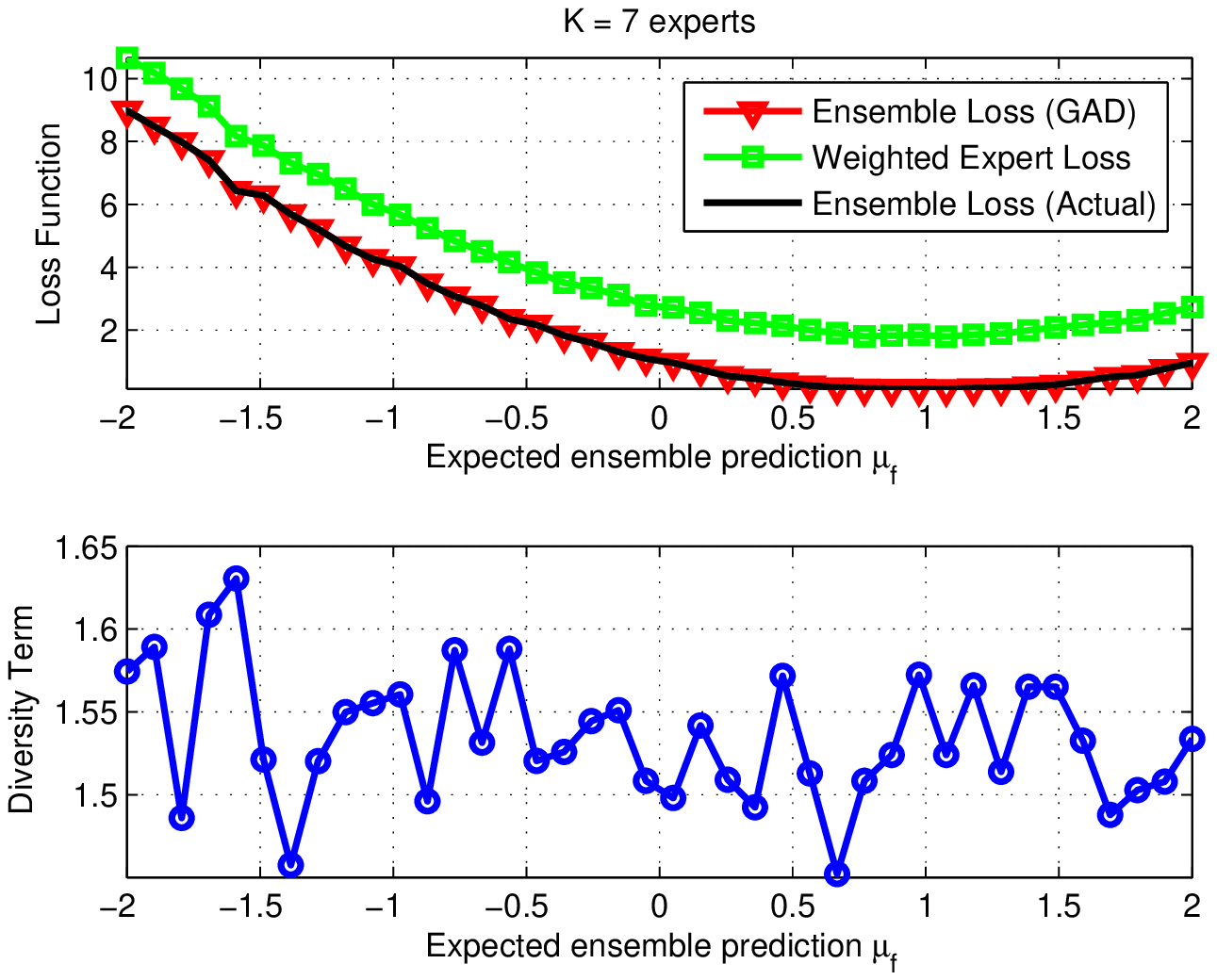}}
 }
	\caption{The top plot in each figure shows the median actual ensemble loss, its GAD approximation and weighted expert loss across $1000$ Monte Carlo samples in an ensemble of $K=3$ and $K=7$ experts for the squared error loss function as a function of expected ensemble prediction $\mu_f$. We used $\sigma_f^2 = 2$. $Y = 1$ is the correct label. We also show the median diversity term for the same setup in the bottom plot.}
 \label{fig:sqr_loss_decomp}
\end{figure}

\begin{figure}[h]
 \centering
 \mbox{
 \subfigure{\includegraphics[width=0.5\textwidth]{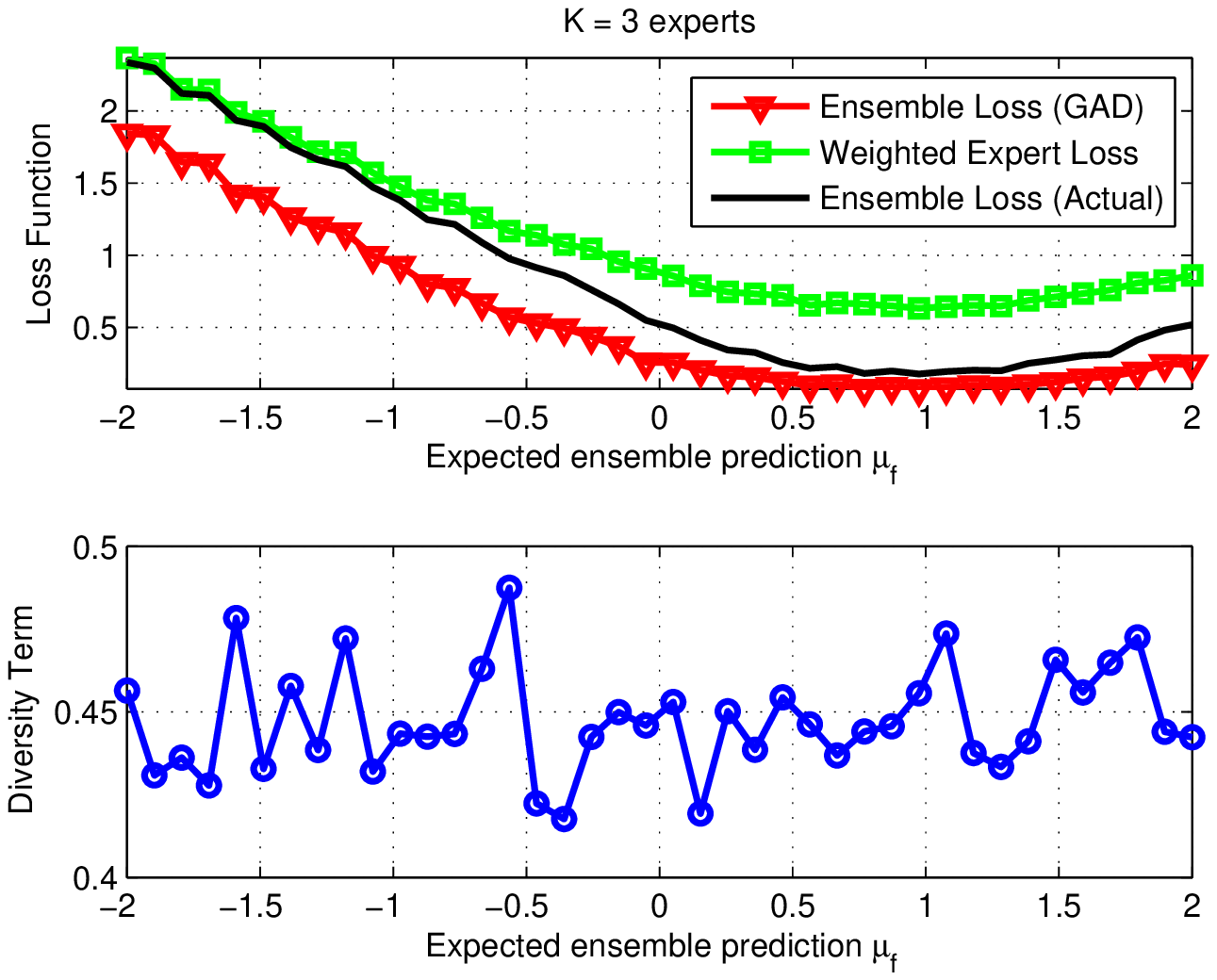}}
 \quad
 \subfigure{\includegraphics[width=0.5\textwidth]{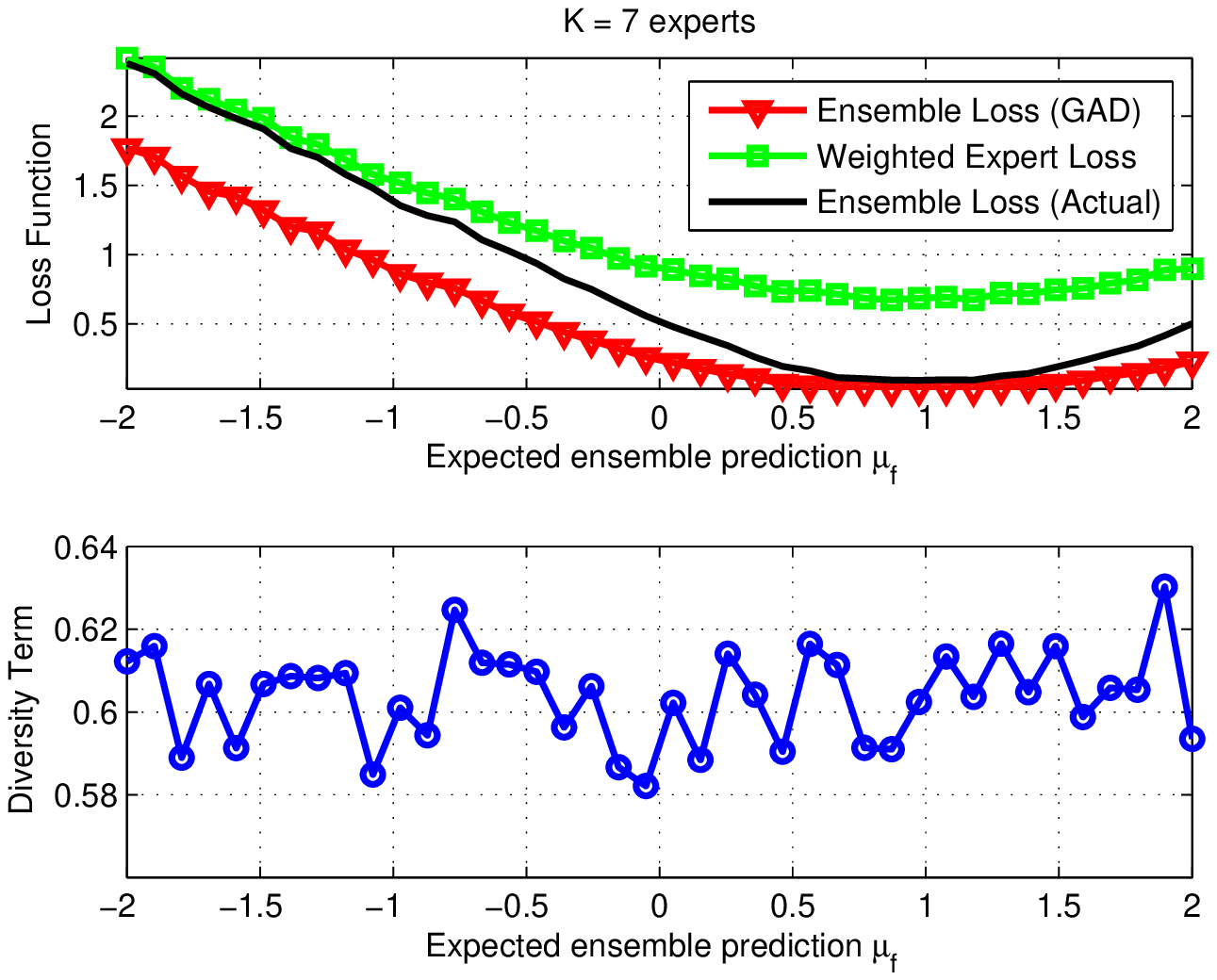}}
 }
	\caption{The top plot in each figure shows the median actual ensemble loss, its GAD approximation and weighted expert loss across $1000$ Monte Carlo samples in an ensemble of $K=3$ and $K=7$ experts for the smooth absolute error loss function as a function of expected ensemble prediction $\mu_f$. We used $\sigma_f^2 = 2$ and $\epsilon = 0.5$. $Y = 1$ is the correct label. We also show the median diversity term for the same setup in the bottom plot.}
 \label{fig:abs_loss_decomp}
\end{figure}

Figures~\ref{fig:log_loss_decomp}-\ref{fig:hin_loss_decomp} show the plots for the three classification loss functions - logistic, exponential, and smooth hinge ($\epsilon = 0.5$) with true label $Y=1$. $l_\text{GAD}(Y,f)$ provides an accurate approximation to $l(Y,f)$ near the true label $Y = 1$ as was the case for the regression loss functions. However the diversity term is not constant, but unimodal with a peak around the decision boundary $\mu_f = 0$. This is because the experts disagree a lot at the decision boundary which causes high diversity. Diversity reduces as we move away from the decision boundary in both directions. Diversity in GAD is agnostic to the true label and only quantifies the spread of the expert predictions with respect to the given loss function. The weighted expert loss term captures the accuracy of the experts with respect to the true label. Even though all experts agree when they are predicting the incorrect class for $\mu_f < 0$, the overall loss rises due to an increase in the weighted expert loss. 

\begin{figure}[h]
 \centering
 \mbox{
 \subfigure{\includegraphics[width=0.5\textwidth]{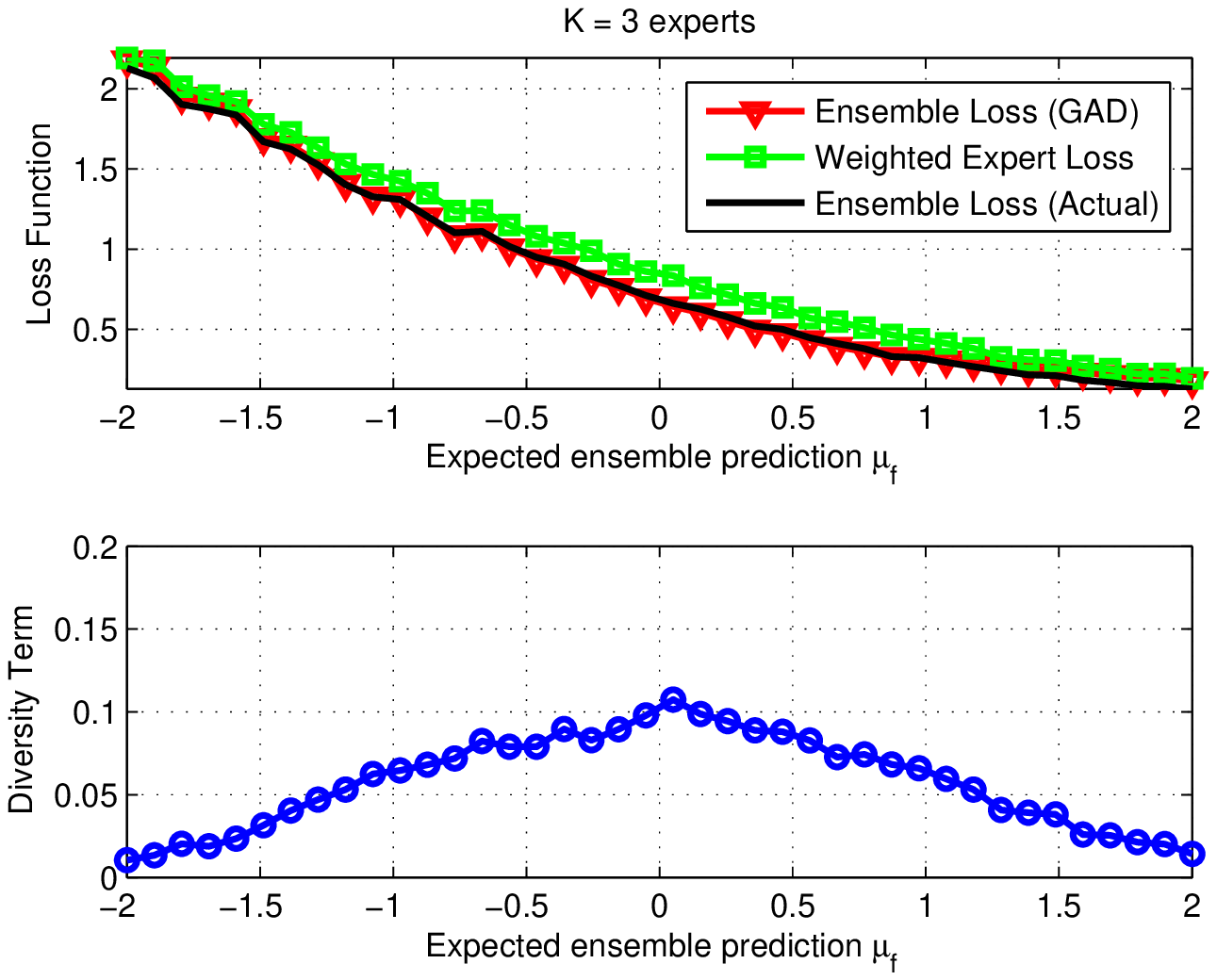}}
 \quad
 \subfigure{\includegraphics[width=0.5\textwidth]{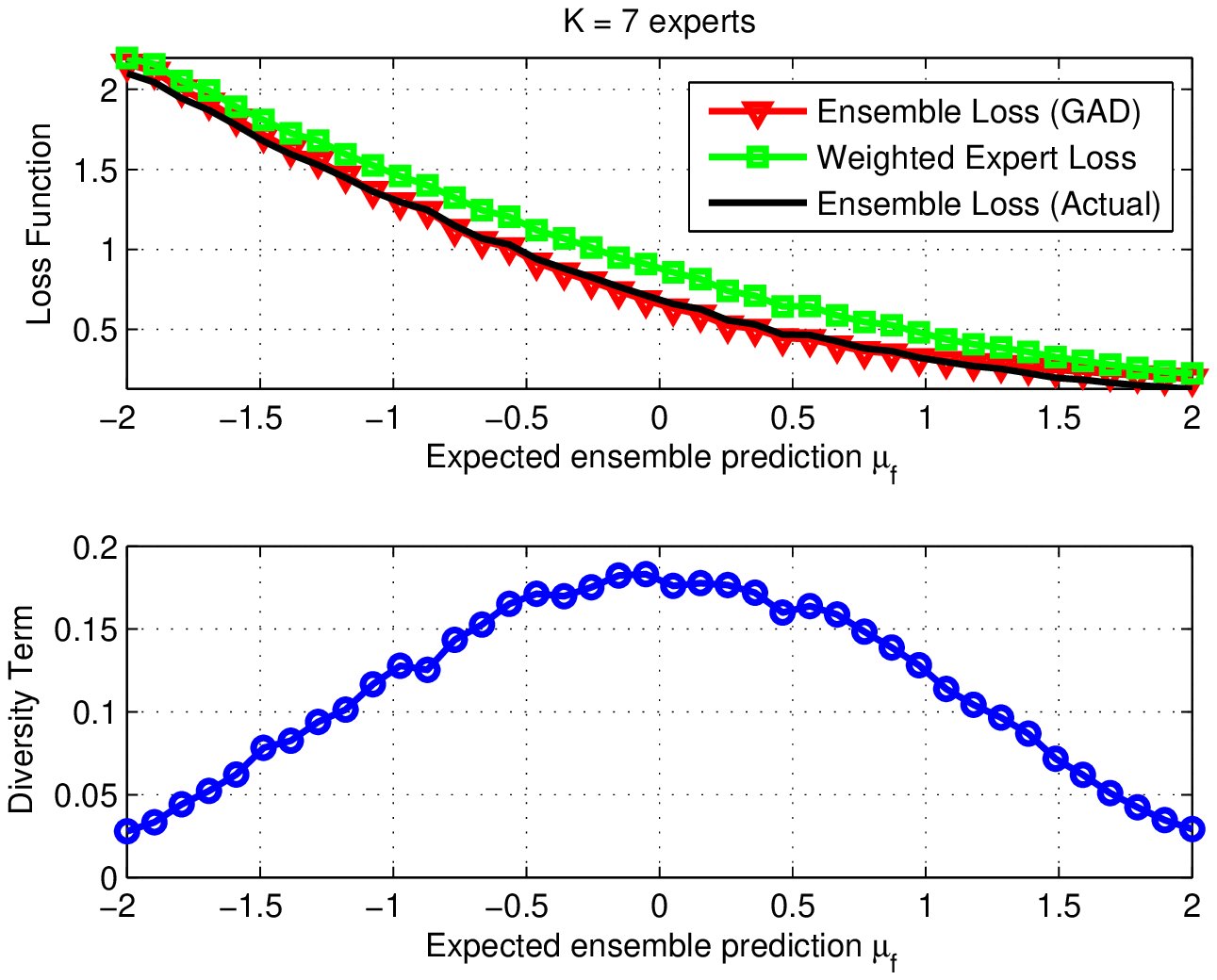}}
 }
	\caption{The top plot in each figure shows the median actual ensemble loss, its GAD approximation and weighted expert loss across $1000$ Monte Carlo samples in an ensemble of $K=3$ and $K=7$ experts for the logistic loss function as a function of expected ensemble prediction $\mu_f$. We used $\sigma_f^2 = 2$. $Y = 1$ is the correct label. We also show the median diversity term for the same setup in the bottom plot.}
 \label{fig:log_loss_decomp}
\end{figure}

\begin{figure}[h]
 \centering
 \mbox{
 \subfigure{\includegraphics[width=0.5\textwidth]{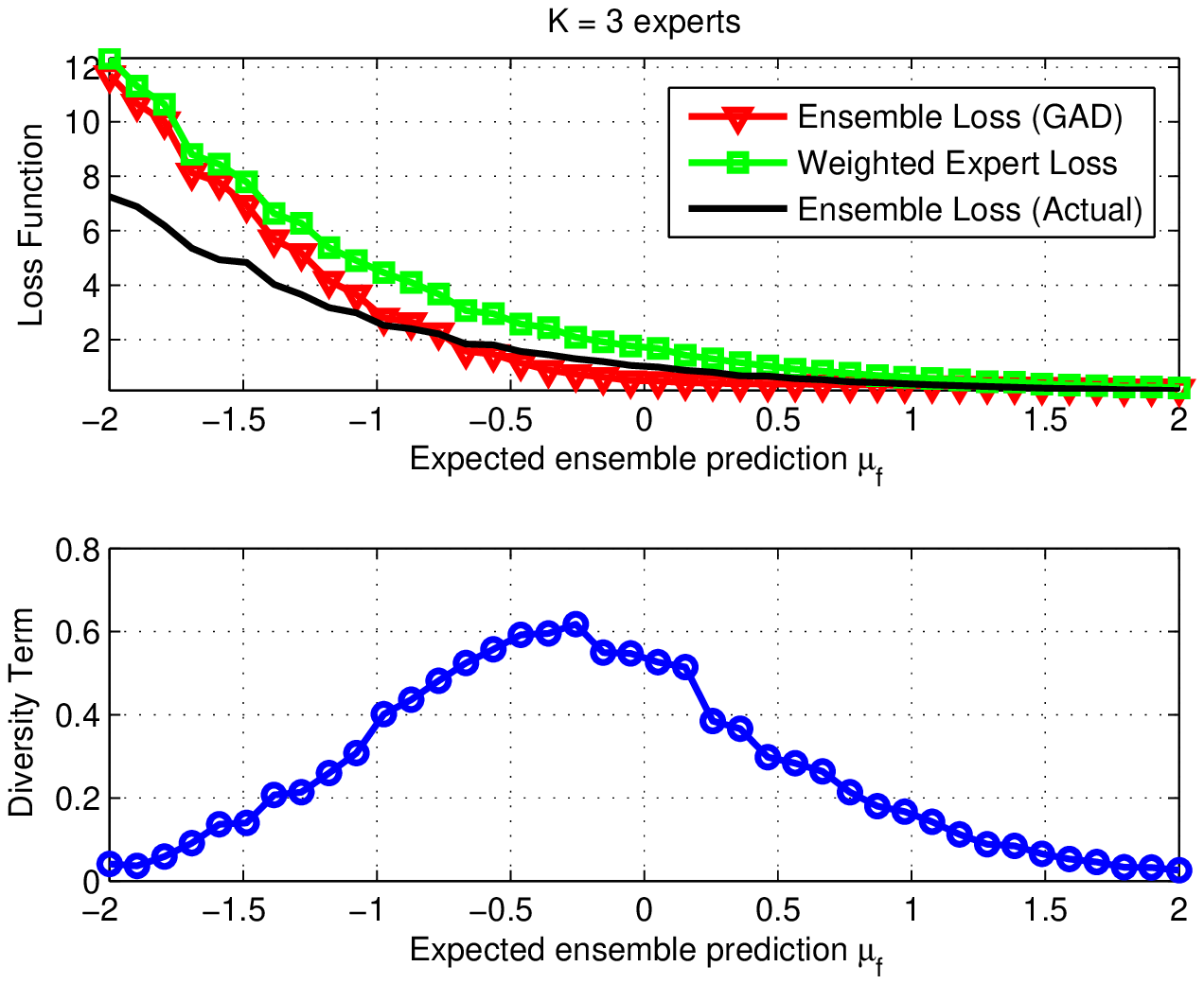}}
 \quad
 \subfigure{\includegraphics[width=0.5\textwidth]{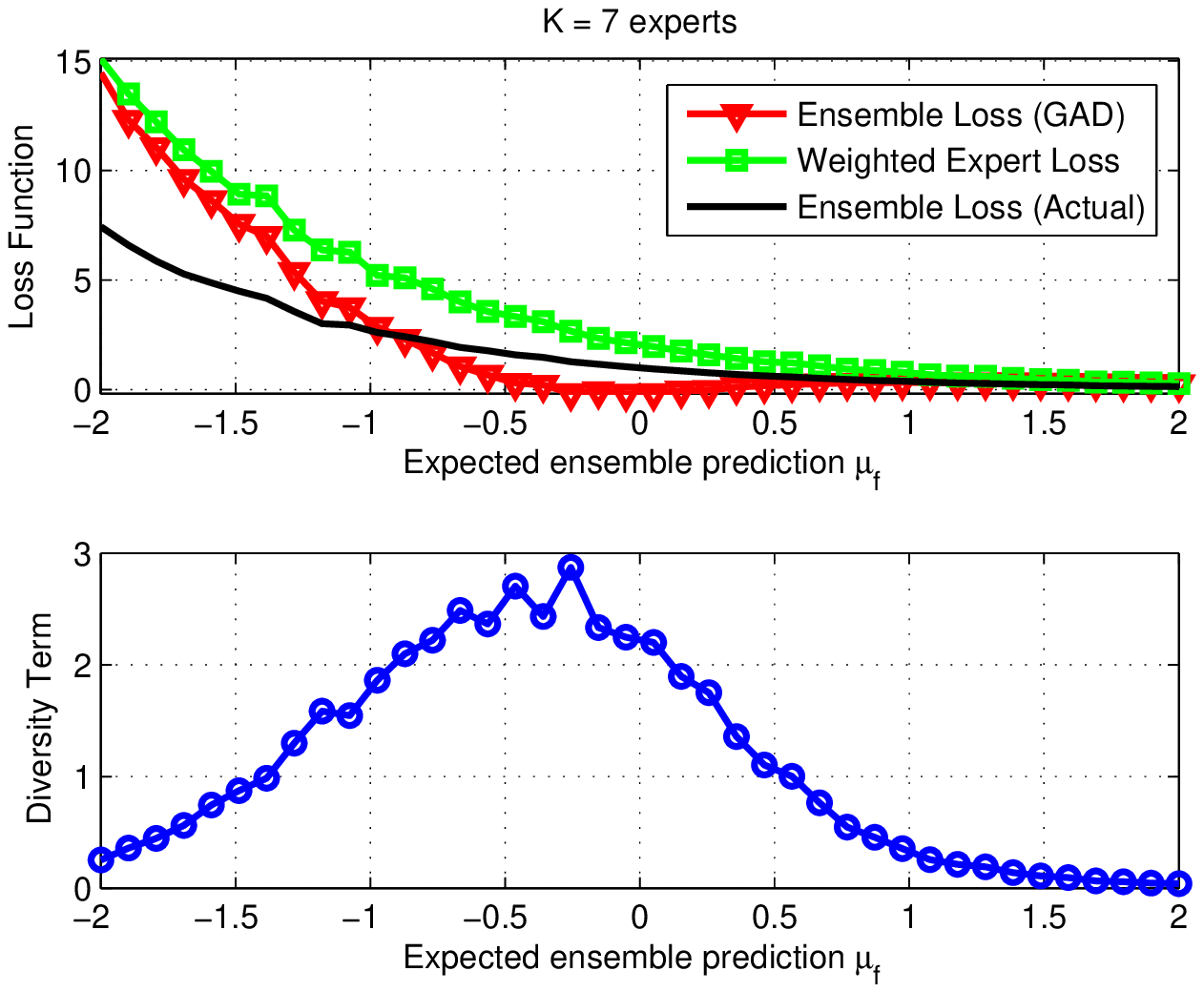}}
 }
	\caption{The top plot in each figure shows the median actual ensemble loss, its GAD approximation and weighted expert loss across $1000$ Monte Carlo samples in an ensemble of $K=3$ and $K=7$ experts for the exponential loss function as a function of expected ensemble prediction $\mu_f$. We used $\sigma_f^2 = 2$. $Y = 1$ is the correct label. We also show the median diversity term for the same setup in the bottom plot.}
 \label{fig:exp_loss_decomp}
\end{figure}

\begin{figure}[h]
 \centering
 \mbox{
 \subfigure{\includegraphics[width=0.5\textwidth]{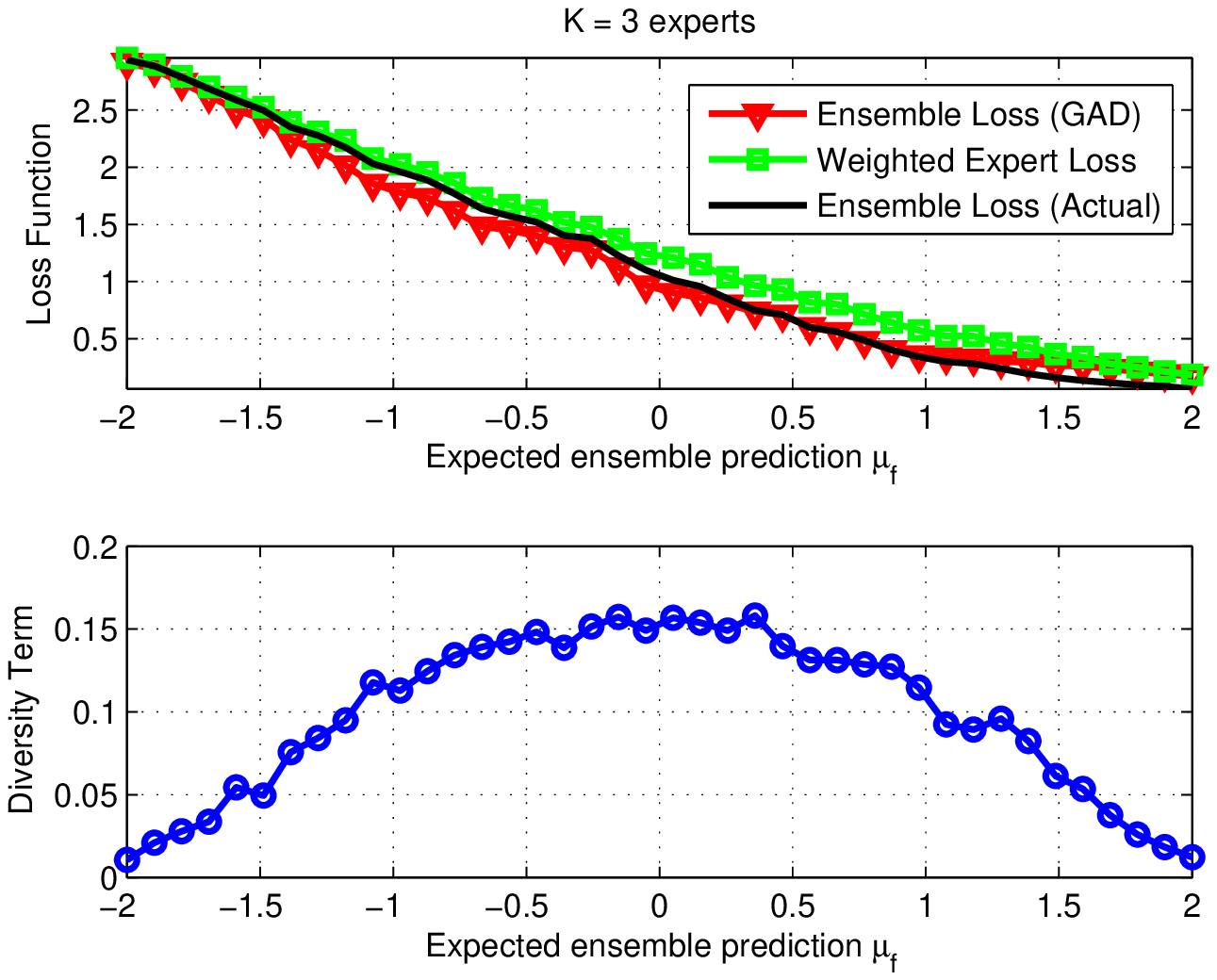}}
 \quad
 \subfigure{\includegraphics[width=0.5\textwidth]{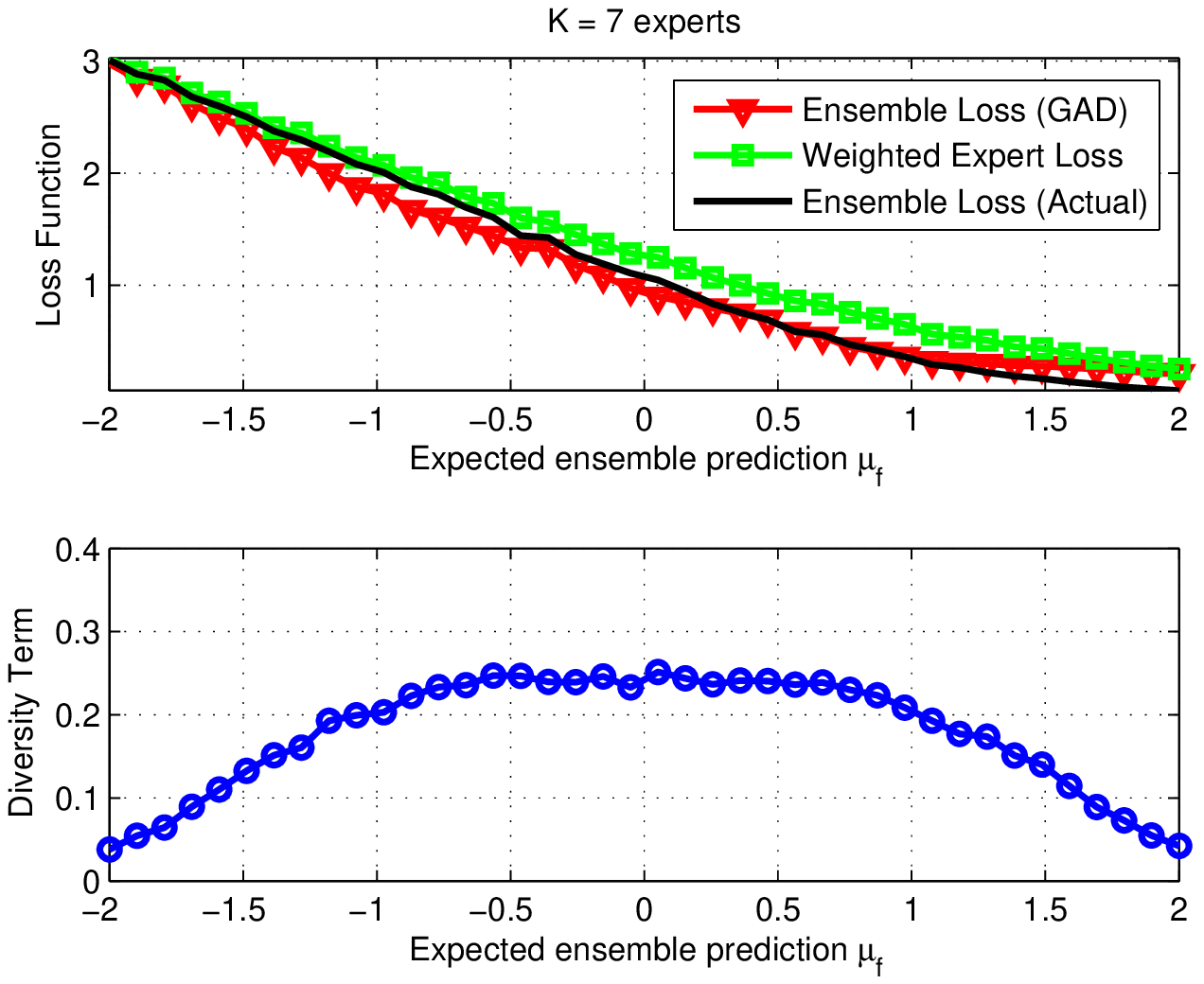}}
 }
	\caption{The top plot in each figure shows the median actual ensemble loss, its GAD approximation and weighted expert loss across $1000$ Monte Carlo samples in an ensemble of $K=3$ and $K=7$ experts for the smooth hinge loss function as a function of expected ensemble prediction $\mu_f$. We used $\sigma_f^2 = 2$ and $\epsilon = 0.5$. $Y = 1$ is the correct label. We also show the median diversity term for the same setup in the bottom plot.}
 \label{fig:hin_loss_decomp}
\end{figure}

\subsection{Accuracy of GAD-Motivated Approximation of Ensemble Loss}
In this subsection, we analyze the error of the approximate GAD ensemble loss $l_\text{GAD}(Y,f)$ in terms of absolute deviation from the true ensemble loss $l(Y,f)$. We also investigate the behavior of the bound on the approximation error $|l(Y,f) - l_\text{GAD}(Y,f)|$ presented in Corollary~\ref{cor:gad}. We used the same experimental setup for simulations as in the previous section. Figures~\ref{fig:abs_approx_err}-\ref{fig:hin_approx_err} show the plots of the approximation error using $l_\text{GAD}(Y,f)$ and the weighted expert loss $l_\text{WGT}(Y,f)$ for various loss functions discussed previously. We did not consider the squared error loss function because GAD reduces to AD and we get $0$ absolute error.

Figures~\ref{fig:abs_approx_err}-\ref{fig:hin_approx_err} show that the GAD approximation $l_\text{GAD}(Y,f)$ (red curve) always provides significantly lower approximation error rate than the weighted expert loss $l_\text{WGT}(Y,f)$ (green curve) when $\mu_f$ is close to the true label $Y = 1$. This is because the second order Taylor series expansion used in the GAD theorem's proof is accurate when the expert predictions are close to the true label. We also note that the bound on the approximation error $|l(Y,f) - l_\text{GAD}(Y,f)|$ (blue curve) follows the general trend of the error but is not very tight.

\begin{figure}[h]
 \centering
 \mbox{
 \subfigure{\includegraphics[width=0.5\textwidth]{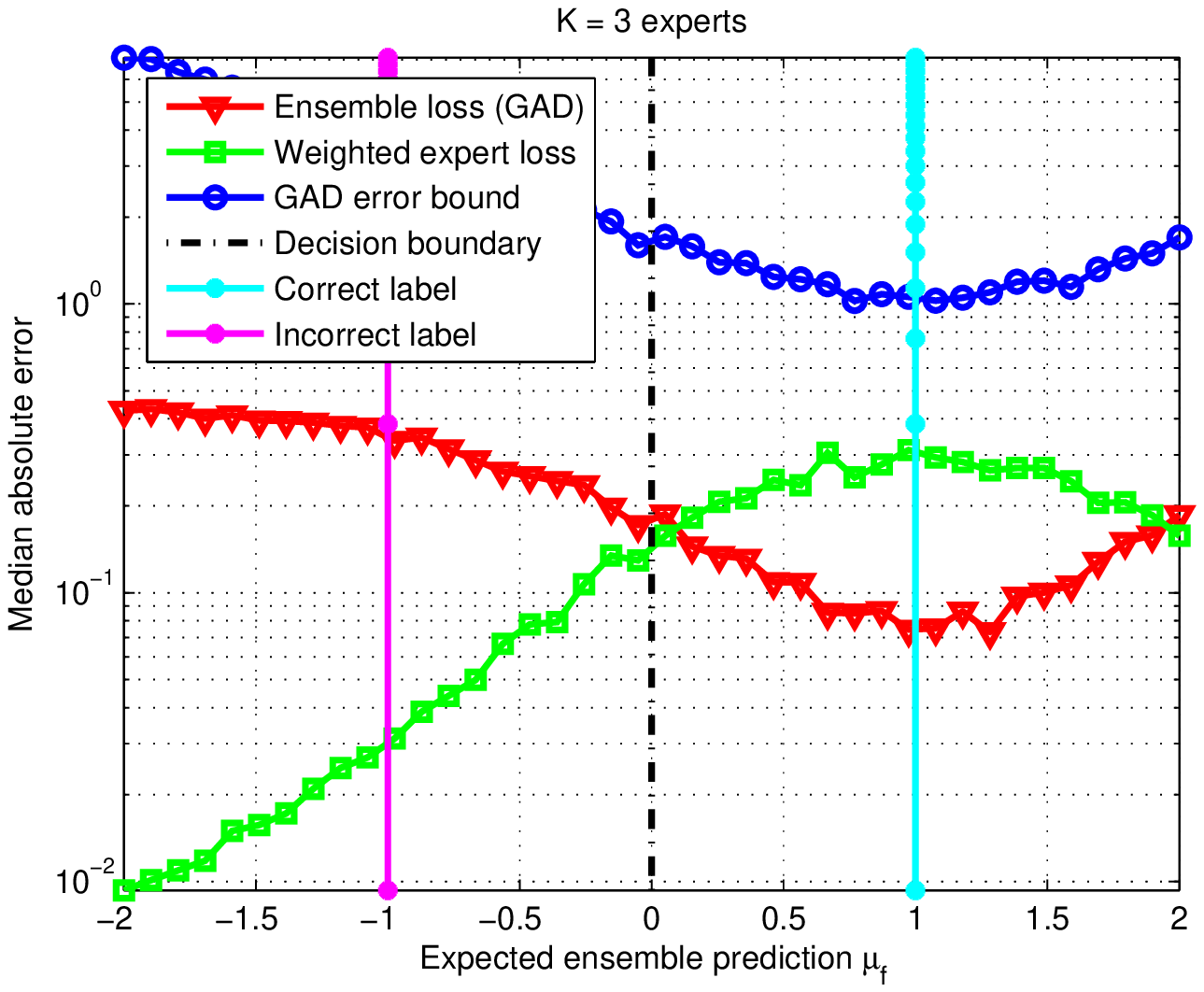}}
 \quad
 \subfigure{\includegraphics[width=0.5\textwidth]{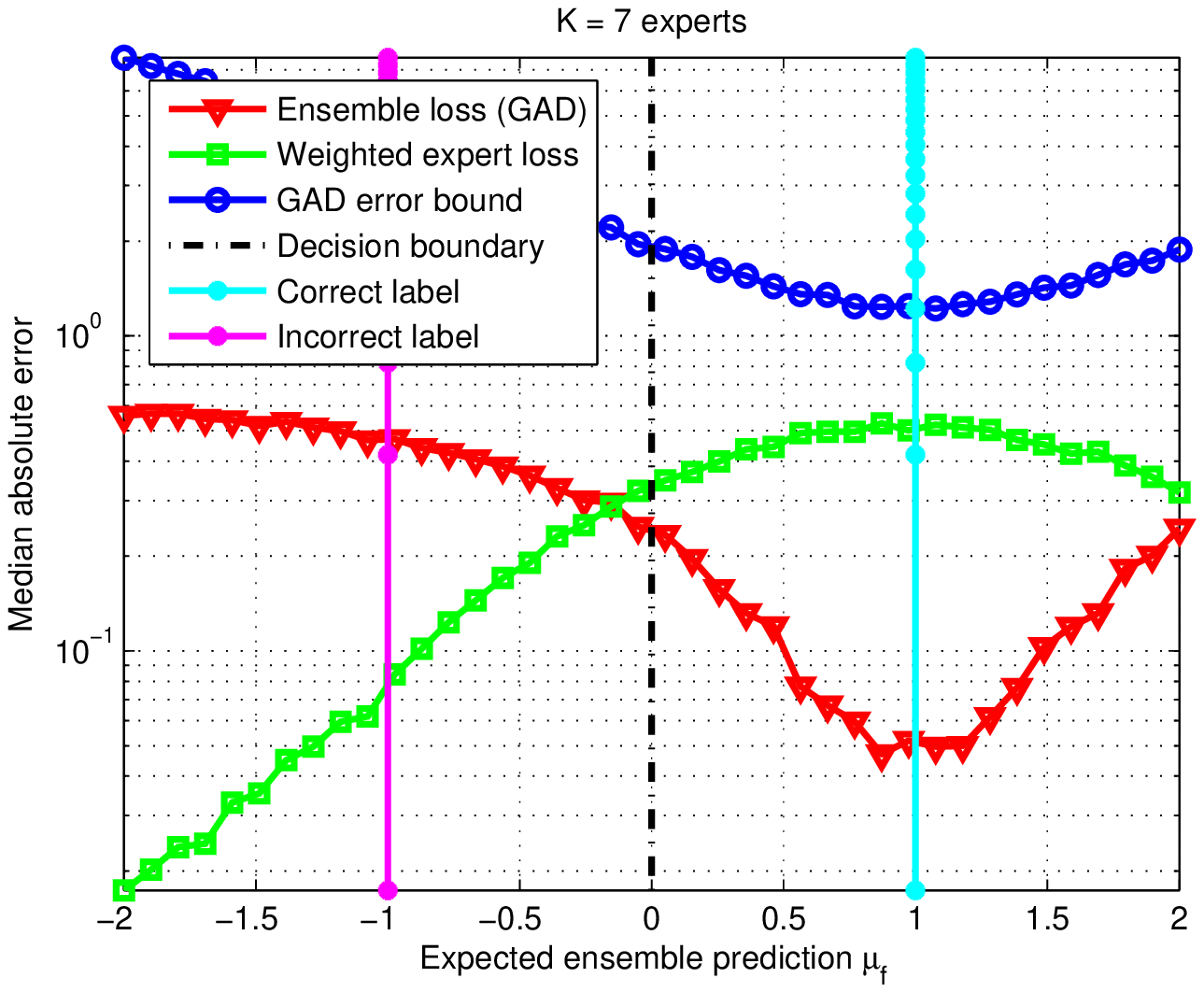}}
 }
 \caption{Median absolute approximation error and error bound across $1000$ Monte Carlo samples in an ensemble of $K=3$ and $K=7$ experts for smooth absolute error loss function as a function of expected ensemble prediction $\mu_f$. We used $\sigma_f^2 = 2$, $\epsilon = 0.5$ and $Y = 1$ as the correct label.}
 \label{fig:abs_approx_err}
\end{figure}

\begin{figure}[h]
 \centering
 \mbox{
 \subfigure{\includegraphics[width=0.5\textwidth]{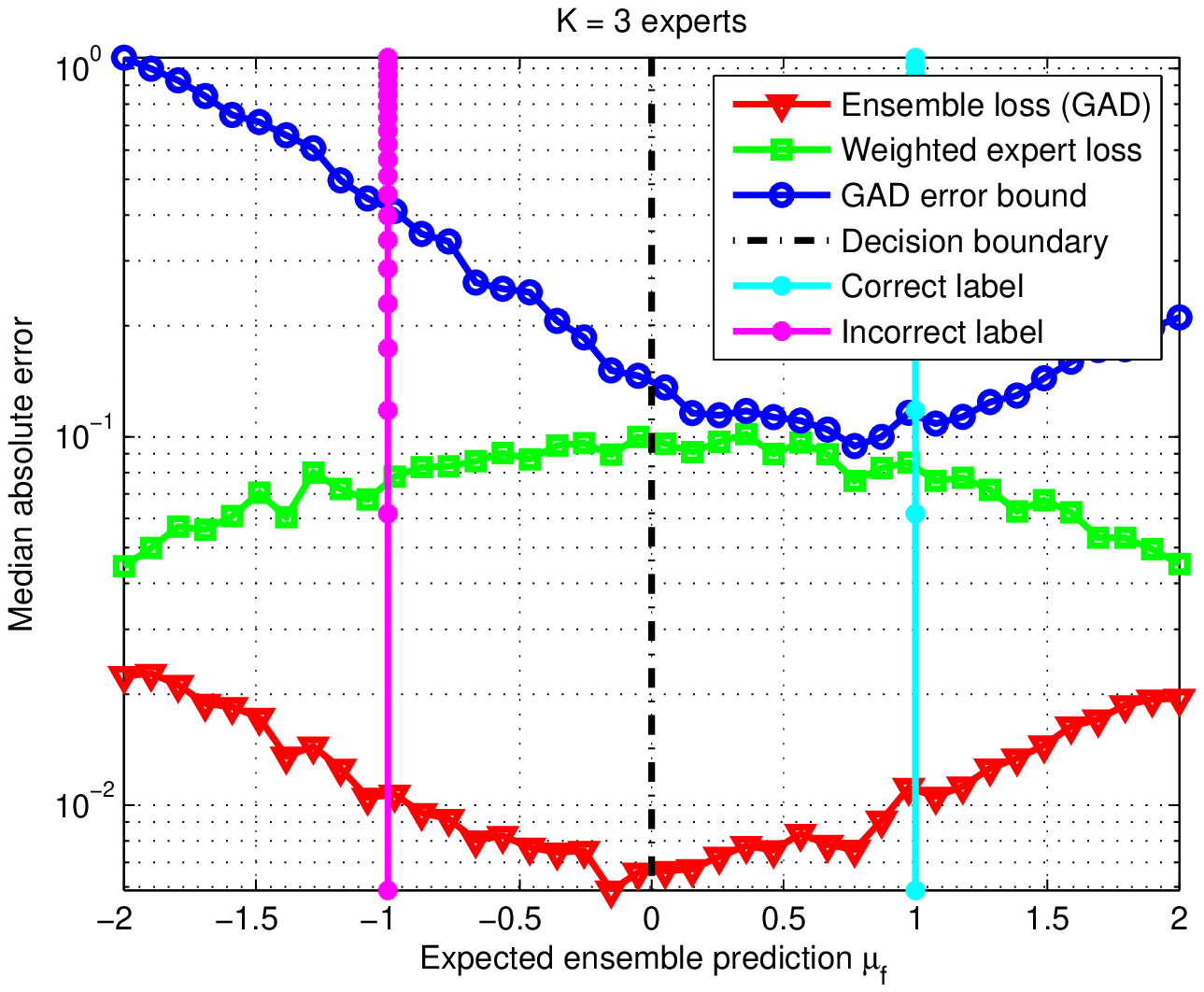}}
 \quad
 \subfigure{\includegraphics[width=0.5\textwidth]{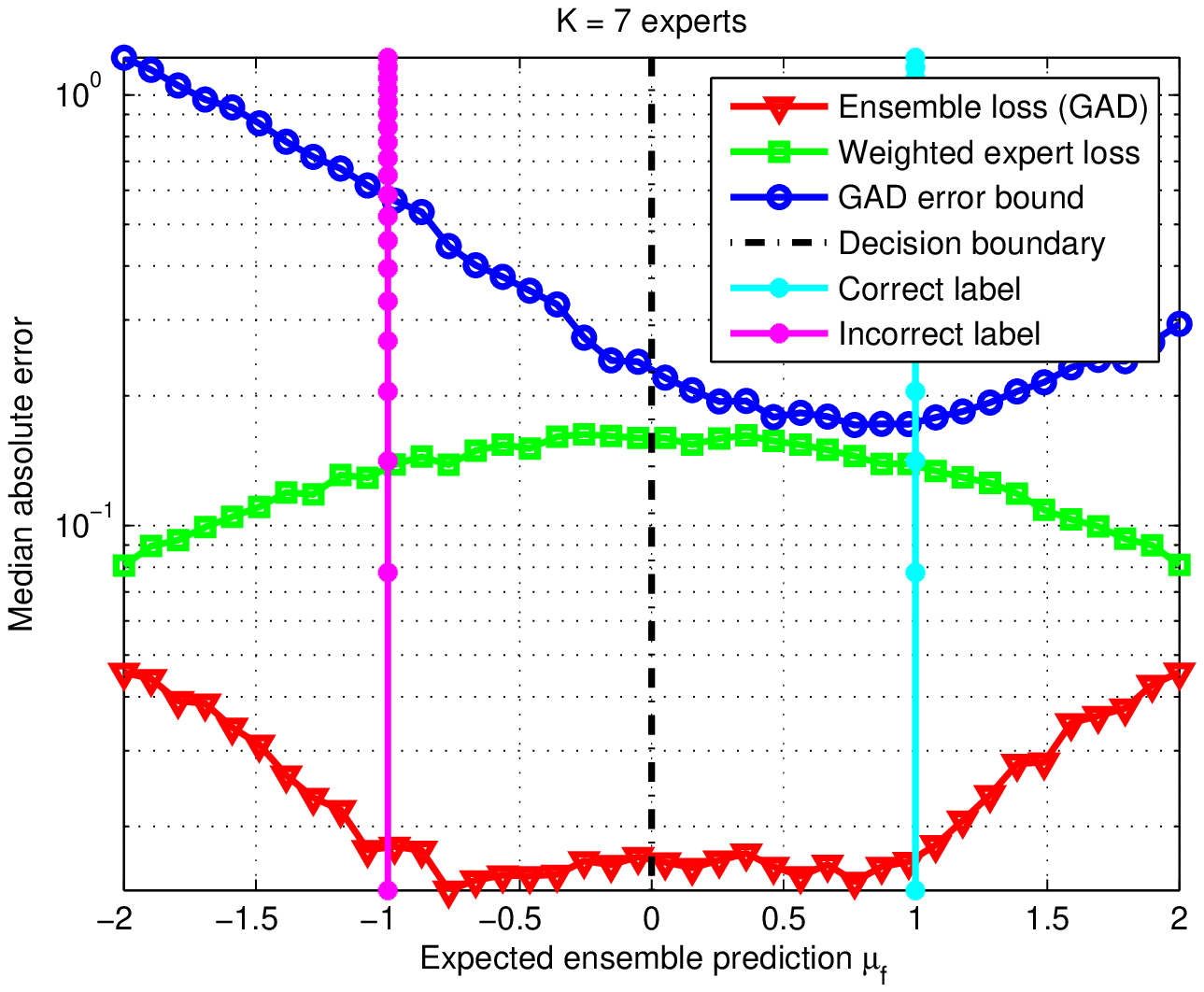}}
 }
 \caption{Median absolute approximation error and error bound across $1000$ Monte Carlo samples in an ensemble of $K=3$ and $K=7$ experts for logistic loss function as a function of expected ensemble prediction $\mu_f$. We used $\sigma_f^2 = 2$ and $Y = 1$ as the correct label.}
 \label{fig:log_approx_err}
\end{figure}

\begin{figure}[h]
 \centering
 \mbox{
 \subfigure{\includegraphics[width=0.5\textwidth]{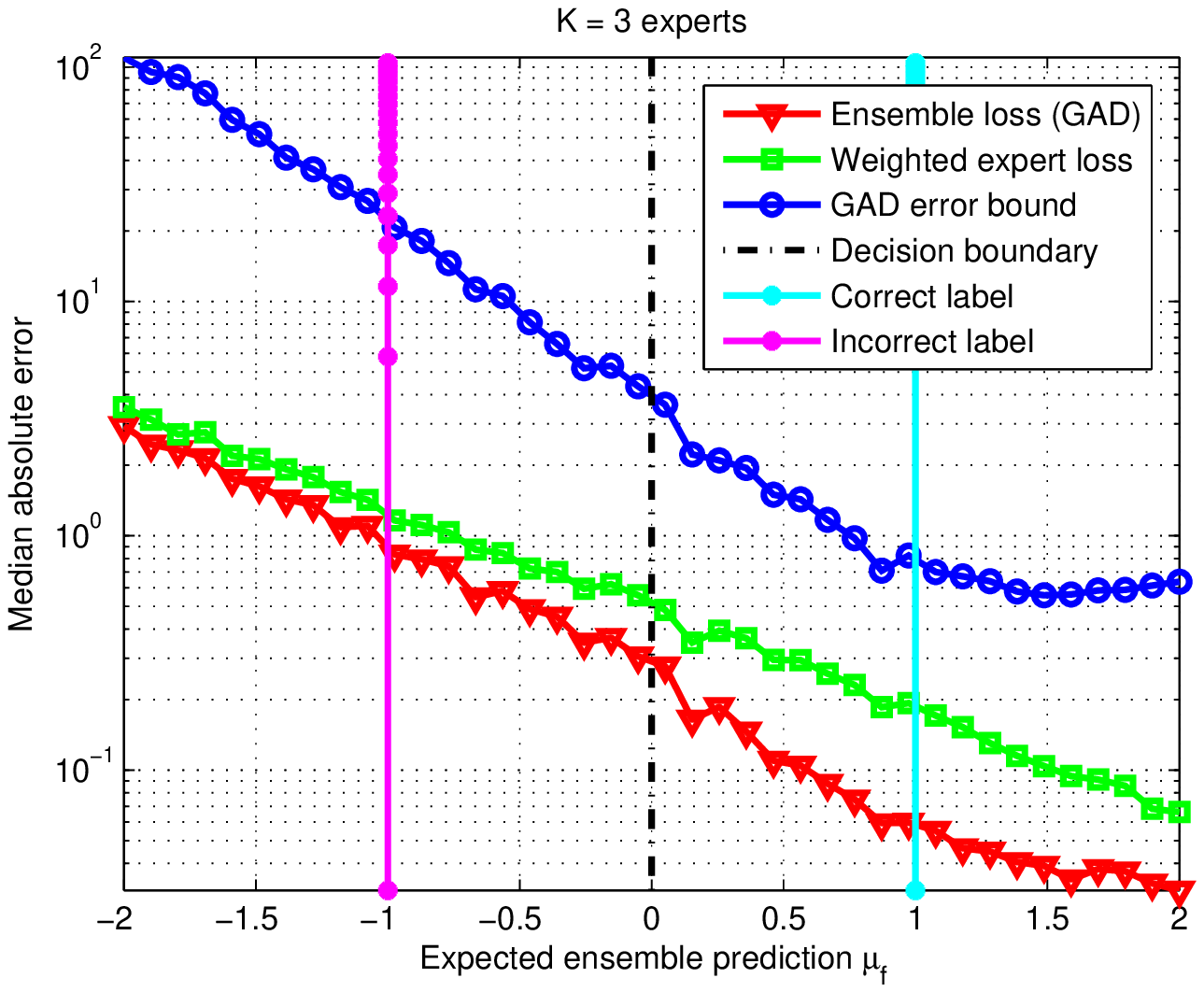}}
 \quad
 \subfigure{\includegraphics[width=0.5\textwidth]{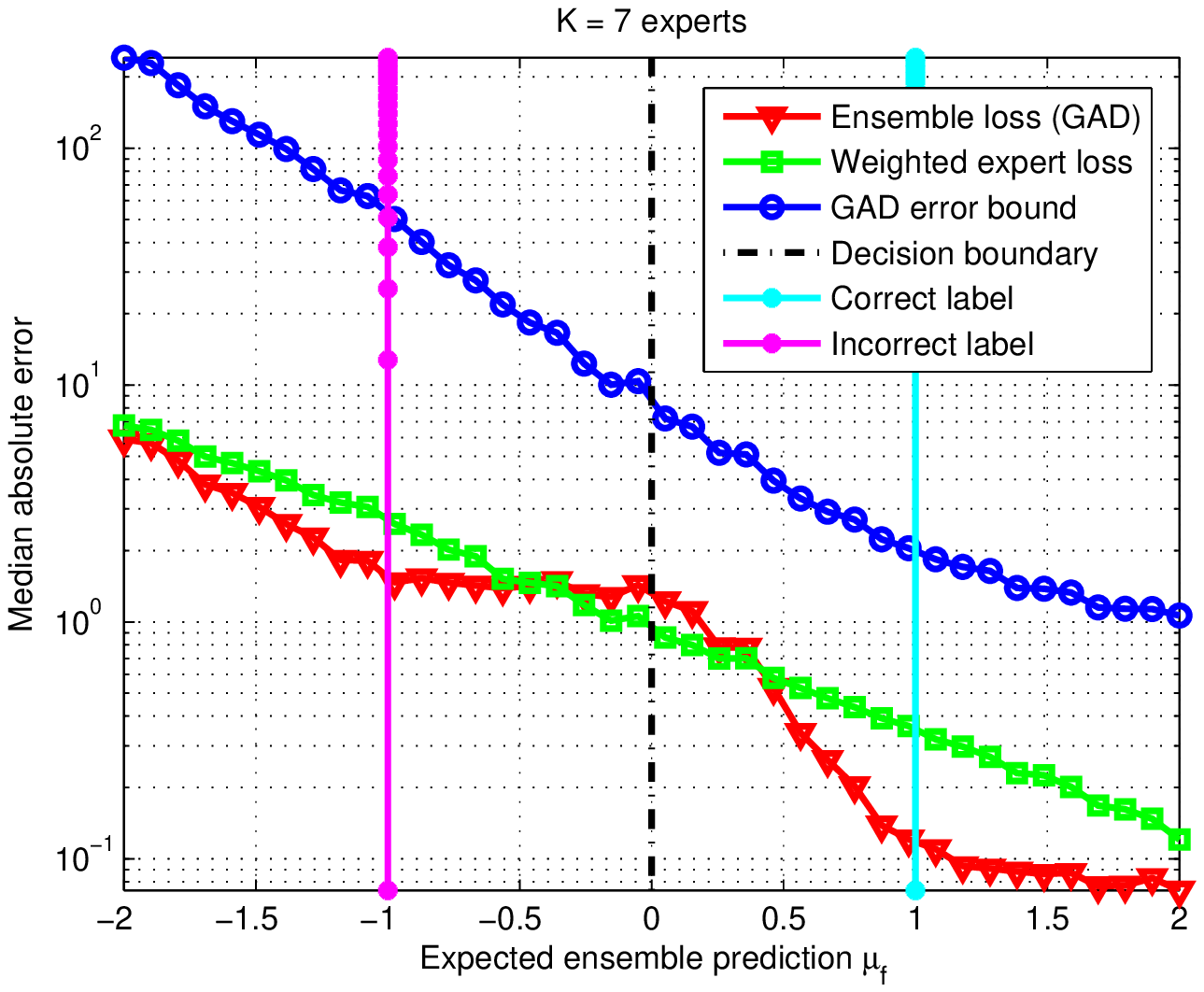}}
 }
 \caption{Median absolute approximation error and error bound across $1000$ Monte Carlo samples in an ensemble of $K=3$ and $K=7$ experts for exponential loss function as a function of expected ensemble prediction $\mu_f$. We used $\sigma_f^2 = 2$ and $Y = 1$ as the correct label.}
 \label{fig:exp_approx_err}
\end{figure}

\begin{figure}[h]
 \centering
 \mbox{
 \subfigure{\includegraphics[width=0.5\textwidth]{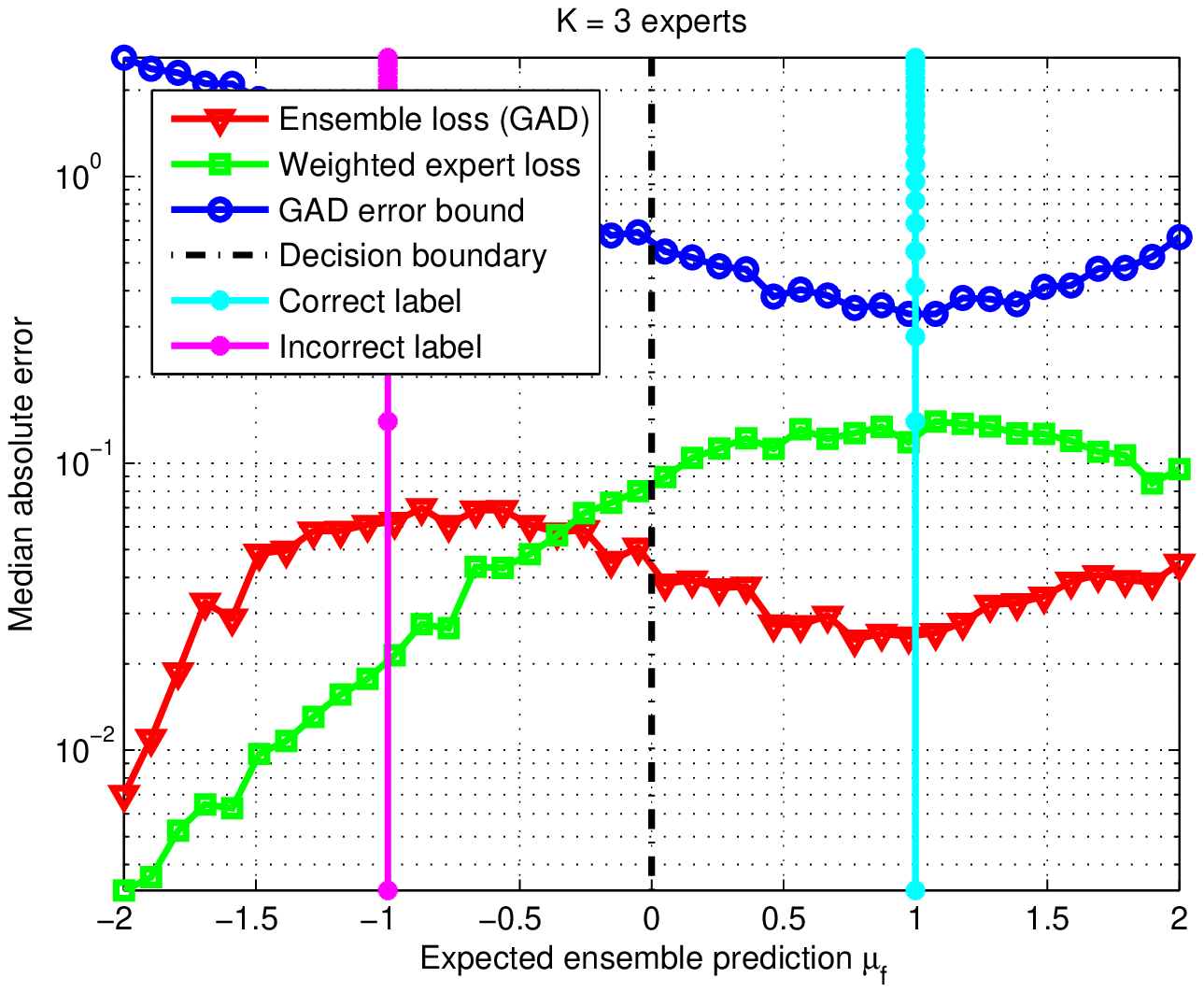}}
 \quad
 \subfigure{\includegraphics[width=0.5\textwidth]{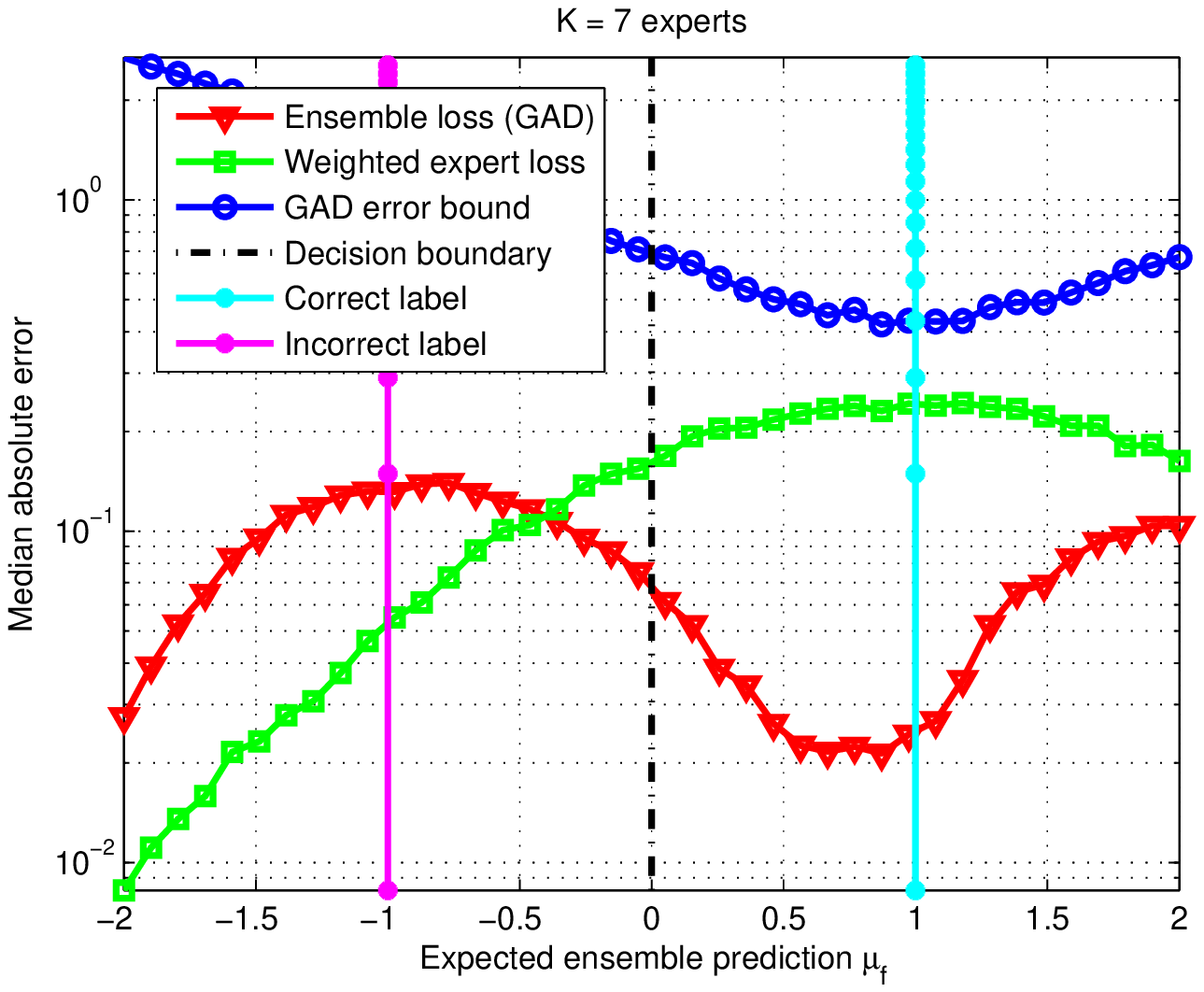}}
 }
 \caption{Median absolute approximation error and error bound across $1000$ Monte Carlo samples in an ensemble of $K=3$ and $K=7$ experts for smooth hinge loss function as a function of expected ensemble prediction $\mu_f$. We used $\sigma_f^2 = 2$, $\epsilon = 0.5$, and $Y = 1$ as the correct label.}
 \label{fig:hin_approx_err}
\end{figure}

\subsection{Comparison with Loss Function Approximation Used in Gradient Boosting}
Gradient boosting~\cite{friedman2001greedy} is a popular machine learning algorithm which sequentially trains an ensemble of base learners. Gradient boosting also utilizes a Taylor series expansion for its sequential training. Hence, we devote this subsection to understanding the differences between the loss function approximation used in gradient boosting and GAD.

Consider an ensemble of $K-1$ experts $f_k$, and their linear combination
\begin{align}
	g &= \sum_{k=1}^{K-1} v_k f_k \;
\end{align}
to generate the ensemble prediction $g$. Gradient boosting does not require the coefficients $\{v_k\}_{k=1}^{K-1}$ to be convex weights. Now if we add a new expert $f_K$ to $g$ with a weight $v_K$, the loss of the new ensemble becomes $l(Y,g + v_K f_K)$. Since gradient boosting estimates $v_K$ and $f_K$ given estimates of $\{v_k,f_k\}_{k=1}^{K-1}$, it assumes that $v_K f_K$ is close to $0$. In other words, it assumes that the new base learner $f_K$ is weak and contributes only that information which has not been learned by the current ensemble. The new ensemble's loss is therefore approximated by using a Taylor series expansion around $v_Kf_K = 0$. Assuming the loss function to be convex, we can write the following first order Taylor series expansion:
\begin{align}\label{eq:gboost}
	l(Y,g + v_K f_K) &\leq l(Y,g) + v_K f_K l'(Y,g) = l_\text{GB}(Y,f) \;.
\end{align}
Minimizing the above upper bound with respect to $v_K$ and $f_K$ is equivalent to minimizing $v_K f_K l'(Y,g)$, or maximizing the correlation between $v_K f_K$ and the negative loss function gradient $-l'(Y,g)$. This is the central idea used in training an ensemble using gradient boosting.

The above Taylor series expansion highlights the key differences between gradient boosting and GAD. First, the loss function upper bound used in gradient boosting is a means to perform sequential training of an ensemble of weak experts. Each new expert adds only incremental information to the ensemble, but is insufficiently trained to predict the target variables on its own. This reduces the utility of the above approximation in (\ref{eq:gboost}) in situations when the individual experts are themselves strong. This arises when, for example, the experts have been trained on different feature sets, data sets, utilize different functional forms, or have not been trained using gradient boosting. Second, the GAD approximation $l_\text{GAD}(Y,f)$ provides an intuitive decomposition of the ensemble loss into the weighted expert loss $l_\text{WGT}(Y,f)$ and the diversity $d(f_1,\ldots,f_K)$ which measures the spread of the expert predictions about $f$. Gradient boosting does not offer such an intuitive decomposition.

Figure~\ref{fig:grad_boost_compare_var0p1} shows the median approximation error for $l_\text{GAD}(Y,f)$ and $l_\text{GB}(Y,f)$ using the exponential loss function. We observe that gradient boosting has minimum error when the ensemble mean $\mu_f$ is near the decision boundary because $f_K \approx 0$. However, the approximation becomes poor as we move away from the decision boundary. $l_\text{GAD}(Y,f)$ provides a good approximation around the true label $Y=1$ as noted in the previous section. Thus the two loss functions $l_\text{GAD}(Y,f)$ and $l_\text{GB}(Y,f)$ provide complementary regions of low approximation error. A similar trend is observed for the other loss functions as well. 

\begin{figure}[h]
 \centering
 \mbox{
 \subfigure{\includegraphics[width=0.5\textwidth]{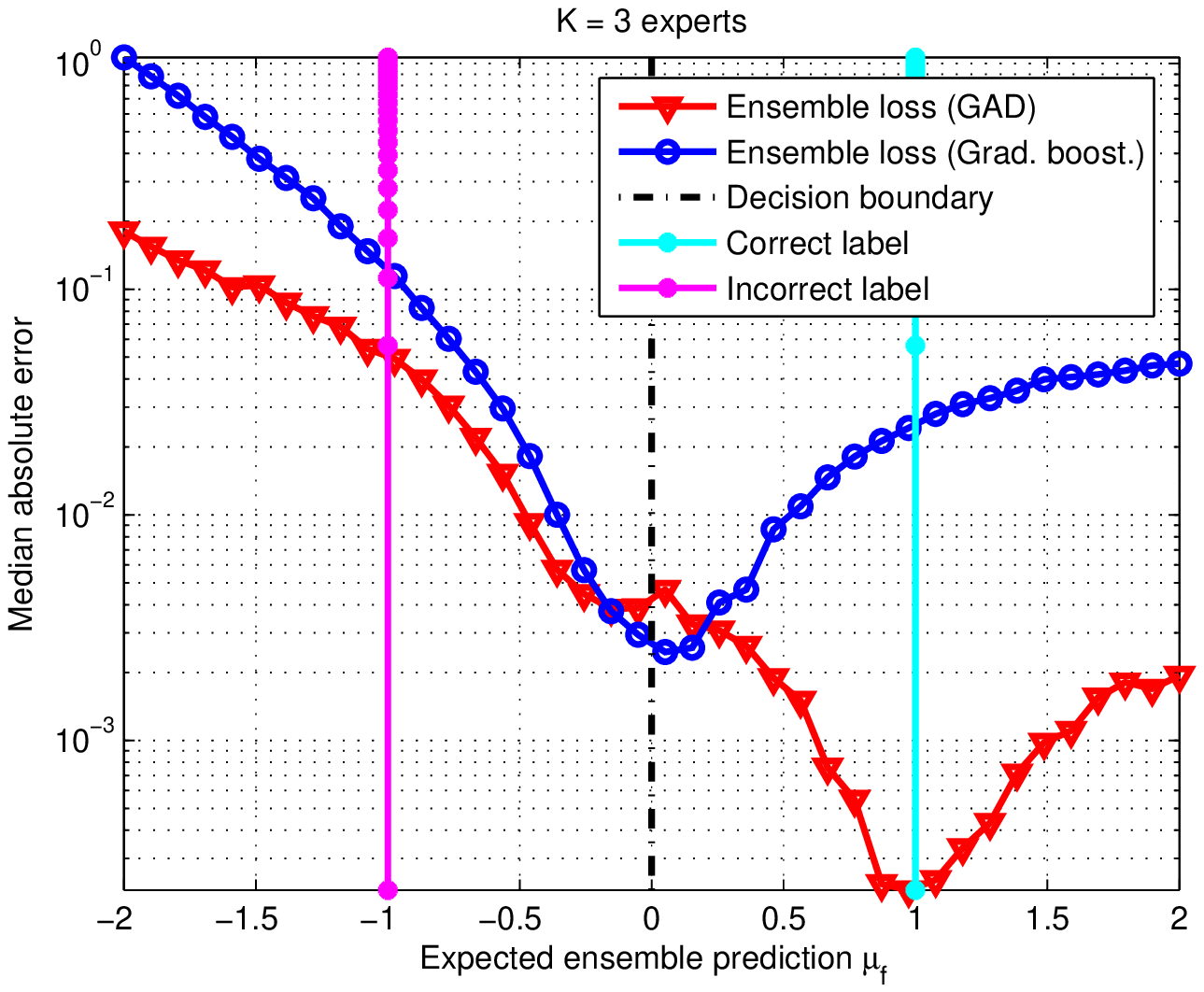}}
 \quad
 \subfigure{\includegraphics[width=0.5\textwidth]{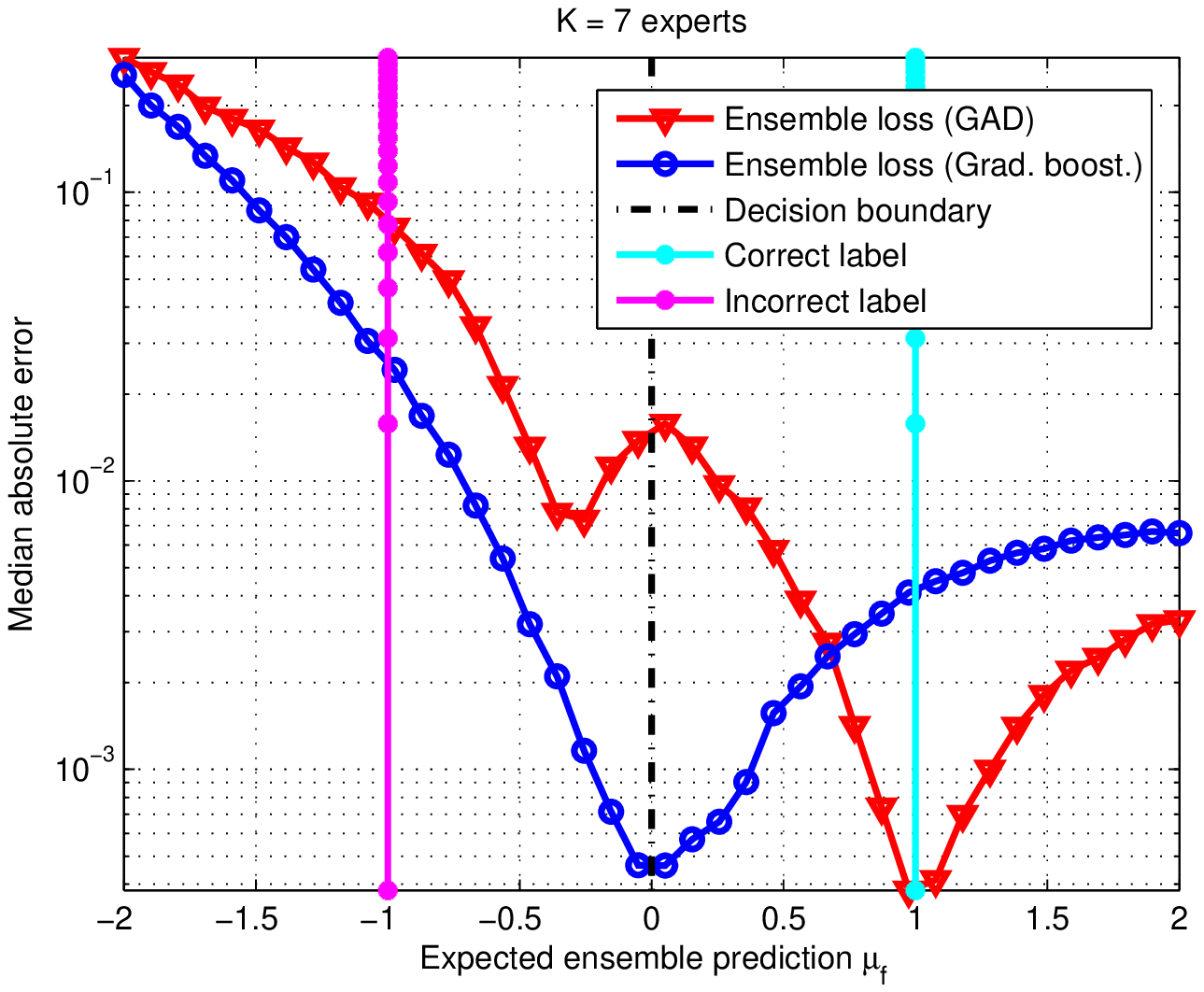}}
 }
 \caption{Median absolute approximation error for GAD and the gradient boosting upper bound across $1000$ Monte Carlo samples in an ensemble of $K=3$ and $K=7$ experts for smooth hinge loss function as a function of expected ensemble prediction $\mu_f$. We used $\sigma_f^2 = 0.1$ and $Y = 1$ as the correct label.}
 \label{fig:grad_boost_compare_var0p1}
\end{figure}

\section{Experiments on GAD with Standard Machine Learning Tasks}\label{sec:expts}
The previous sections presented empirical analysis of the GAD theorem based on simulations. This section presents experiments on some real-world data sets which will reveal the utility of the GAD theorem to machine learning problems of interest. We used five data sets from the UCI Machine Learning Repository~\cite{uci} as listed in Table~\ref{tab:uci_datasets} for the experiments. We conducted two sets of experiments using the UCI data sets. These experiments mimic common scenarios usually encountered by researchers while training systems with multiple classifiers or regressors. 

The first class of experiments tries to understand diversity and its impact on ensemble performance in case the ensemble consists of different classifiers and regressors trained on the same data set. We trained $3$ classifiers and $3$ regressors for each data set. We used logistic regression, linear support vector machine (SVM) from the Lib-Linear toolkit~\cite{liblinear}, and a homoscedastic linear discriminant analysis (LDA)-based classifier from Matlab for classification. The three linear regressors were trained by minimizing least squares, least absolute deviation, and the Huber loss function. GAD was used to analyse the diversity of the trained classifiers and regressors for each data set. The second experiment considers the situations where the experts are trained on potentially overlapping subsets of instances from a given data set. We used bagging~\cite{breiman1996bagging} for creating the multiple training subsets by sampling instances with replacement. The classifiers and regressors mentioned above were used for these experiments one at a time. 

The evaluation metric of the above experiments is the relative approximation error between the true loss and its approximation:
\begin{align}
	E_x&= \Bigg{|}1 - \frac{l_\text{x}(Y,f(X))}{l(Y,f(X))}\Bigg{|}
\end{align}
where $x$ is one of $\text{GAD}$, $\text{WGT}$ (using weighted loss of ensemble (\ref{eq:wgt_exp_loss})), and $\text{GB}$ (using approximation used in gradient boosting (\ref{eq:gboost})). We assigned equal weights \{$w_k$\} to the experts in all experiments.

Table~\ref{tab:approx_perf_case1} shows the relative absolute error for various classification data sets and loss functions when different experts were trained on the same data set. We observe that the GAD approximation provides the lowest median absolute error for all cases. This result is statistically significant at $\alpha = 0.01$ level using the paired t-test. It is often an order of magnitude better than the other two approximations. Table~\ref{tab:approx_perf_case2} shows the relative absolute error when one expert was trained on three versions of the same data set created by sampling with replacement. We used logistic regression for the classification loss functions and least squares linear regression for the regression loss functions. Table~\ref{tab:approx_perf_case2} shows that the GAD approximation gives the lowest error for all cases except for the Wine Quality data set with smooth absolute error loss function.

\begin{table}
	\centering
		\begin{tabular}{|c|c|c|c|}
		\hline
		Data set & Target set & No. of instances & No. of features \\ \hline \hline
		Magic Gamma Telescope~\cite{magic_gamma} & $\{-1,1\}$ & 19020 & 10 \\ \hline
		Pima Indians Diabetes~\cite{pima_indians} & $\{-1,1\}$ & 768 & 8 \\ \hline
		Abalone~\cite{abalone} & $\{-1,1\}, \mathbb{Z}^+$ & 4177 & 7 \\ \hline
		Parkinson's Disease~\cite{parkinsons} & $\mathbb{R}^+$ & 5875 & 20 \\ \hline
		Wine Quality~\cite{wine} & $\{0,\ldots,10\}$ & 6497 & 11 \\ \hline			
		\end{tabular}
	\caption{Description of various UCI Machine Learning Repository data sets used for experiments on GAD. All data sets with $\{-1,1\}$ as the target set were used for training binary classifiers. Others were used for training regressors. Abalone was used for binary classification as well by first thresholding the target variable at $10$.}
	\label{tab:uci_datasets}
\end{table}

\begin{table}
	\centering
		\begin{tabular}{|c|c|c|c|}
		\hline
	  Data set & $E_\text{GAD}$ & $E_\text{WGT}$ & $E_\text{GB}$ \\ \hline \hline
		\multicolumn{4}{|c|}{Logistic Loss} \\ \hline
		Magic Gamma Telescope~\cite{magic_gamma} & \textbf{1.3e-2} & 6.5e-2 & 6.6e-2 \\ \hline
		Pima Indians Diabetes~\cite{pima_indians} & \textbf{7.0e-3} & 2.9e-2 & 2.3e-2 \\ \hline
		Abalone~\cite{abalone} & \textbf{9.0e-3} & 4.4e-1 & 4.2e-1 \\ \hline \hline
		\multicolumn{4}{|c|}{Exponential Loss} \\ \hline
		Magic Gamma Telescope~\cite{magic_gamma} & \textbf{3.5e-2} & 1.3e-1 & 1.4e-1 \\ \hline
		Pima Indians Diabetes~\cite{pima_indians} & \textbf{1.5e-2} & 6.8e-2 & 6.1e-2 \\ \hline
		Abalone~\cite{abalone} & \textbf{2.2e-2} & 9.2e-2 & 9.2e-2 \\ \hline \hline
		\multicolumn{4}{|c|}{Smooth Hinge Loss ($\epsilon = 0.5$)} \\ \hline
		Magic Gamma Telescope~\cite{magic_gamma} & \textbf{2.2e-2} & 1.6e-1 & 1.4e-1 \\ \hline
		Pima Indians Diabetes~\cite{pima_indians} & \textbf{7.0e-3} & 4.9e-2 & 3.1e-2 \\ \hline
		Abalone~\cite{abalone} & \textbf{1.2e-2} & 8.9e-2 & 6.8e-2 \\ \hline \hline
		\multicolumn{4}{|c|}{Squared Error Loss} \\ \hline 
		Parkinson's Disease~\cite{parkinsons} & \textbf{0} & 1.7e-1 & 1.7 \\ \hline
		Wine Quality~\cite{wine} & \textbf{0} & 1.9e-1 & 1.5e-1 \\ \hline
		Abalone~\cite{abalone} & \textbf{0} & 2.8e-2 & 1.4e-1 \\ \hline  \hline
		\multicolumn{4}{|c|}{Smooth Absolute Error Loss ($\epsilon = 0.5$)} \\ \hline
		Parkinson's Disease~\cite{parkinsons} & \textbf{1e-2} & 1.4e-1 & 1.3 \\ \hline
		Wine Quality~\cite{wine} & \textbf{9.1e-2} & \textbf{9e-2} & 9.9e-2 \\ \hline
		Abalone~\cite{abalone} & \textbf{7e-3} & 1.9e-2 & 9.9e-2 \\ \hline
		\end{tabular}
	\caption{This table shows the relative absolute error $E_\text{x}$ for various UCI data sets between the ensemble loss and approximation x which is one of GAD, WGT, and GB corresponding to GAD, weighted sum of expert loses, and gradient boosting upper-bound on total loss. The first three loses used an ensemble of three classifiers - logistic regression, linear support vector machine, and Fisher's linear discriminant analysis classifiers. The two last two regressors used three regressors obtained by minimizing squared error, absolute error, and Huber loss function. The GAD approximation has significantly lower error than the other approximations for all cases except the Wine quality data set for the smooth absolute loss.}
	\label{tab:approx_perf_case1}
\end{table}

\begin{table}
	\centering
		\begin{tabular}{|c|c|c|c|}
		\hline
	  Data set & $E_\text{GAD}$ & $E_\text{WGT}$ & $E_\text{GB}$ \\ \hline \hline
		\multicolumn{4}{|c|}{Logistic Loss} \\ \hline
		Magic Gamma Telescope~\cite{magic_gamma} & \textbf{1.76e-4} & 5.38e-4 & 1.07e-1 \\ \hline
		Pima Indians Diabetes~\cite{pima_indians} & \textbf{4.02e-3} & 1.86e-2 & 3.62e-2 \\ \hline
		Abalone~\cite{abalone} & \textbf{1.24e-3} & 3.91e-3 & 5.54e-2 \\ \hline \hline
		\multicolumn{4}{|c|}{Exponential Loss} \\ \hline
		Magic Gamma Telescope~\cite{magic_gamma} & \textbf{4.03e-4} & 9.92e-4 & 1.77e-1 \\ \hline
		Pima Indians Diabetes~\cite{pima_indians} & \textbf{4.22e-3} & 1.61e-2 & 7.19e-2 \\ \hline
		Abalone~\cite{abalone} & \textbf{1.49e-3} & 4.09e-3 & 9.70e-2 \\ \hline \hline
		\multicolumn{4}{|c|}{Smooth Hinge Loss ($\epsilon = 0.5$)} \\ \hline
		Magic Gamma Telescope~\cite{magic_gamma} & \textbf{8.93e-4} & 1.89e-3 & 3.31e-1 \\ \hline
		Pima Indians Diabetes~\cite{pima_indians} & \textbf{1.26e-2} & 4.93e-2 & 5.93e-2 \\ \hline
		Abalone~\cite{abalone} & \textbf{2.14e-3} & 6.18e-3 & 1.46e-1 \\ \hline \hline
		\multicolumn{4}{|c|}{Squared Error Loss} \\ \hline 
		Parkinson's Disease~\cite{parkinsons} & \textbf{0} & 5.5e-3 & 1.6 \\ \hline
		Wine Quality~\cite{wine} & \textbf{0} & 1.5e-2 & 8.7e-2 \\ \hline
		Abalone~\cite{abalone} & \textbf{0} & 9.3e-3 & 1.5e-1 \\ \hline  \hline
		\multicolumn{4}{|c|}{Smooth Absolute Error Loss ($\epsilon = 0.5$)} \\ \hline
		Parkinson's Disease~\cite{parkinsons} & \textbf{5.3e-4} & 7.4e-3 & 1.2 \\ \hline
		Wine Quality~\cite{wine} & 1.2e-2 & \textbf{4.4e-3} & 6.0e-2 \\ \hline
		Abalone~\cite{abalone} & \textbf{1.4e-3} & 3.4e-3 & 8.2e-2 \\ \hline
		\end{tabular}
	\caption{This table shows the relative absolute error $E_\text{x}$ for various UCI data sets between the ensemble loss and approximation x which is one of GAD, WGT, and GB corresponding to GAD, weighted sum of expert loses, and gradient boosting upper-bound on total loss. The first three loses used logistic regression classifiers trained on three data sets created by sampling with replacement (bagged training sets). The last two loses used a linear regressor obtained by minimizing squared error and trained on 3 bagged training sets. The GAD approximation has significantly lower error than the other approximations for all cases except the Wine quality data set for the smooth absolute loss.}
	\label{tab:approx_perf_case2}
\end{table}

These experiments indicate that $l_\text{GAD}(Y,f)$ provides an accurate approximation of the actual ensemble loss $l(Y,f)$. This adds value to the proposed approximation for designing supervised machine learning algorithms, in addition to providing an intuitive definition and explanation of the impact of diversity on ensemble performance.

\section{Conclusion and Future Work}\label{sec:concl}
We presented the generalized ambiguity decomposition (GAD) theorem which explains the link between the diversity of experts in an ensemble and the ensemble's overall performance. The GAD theorem applies to a convex ensemble of arbitrary experts with a second order differentiable loss function. It also provides a data-dependent and loss function-dependent definition of diversity. We applied this theorem to some commonly used classification and regression loss functions and provided a simulation-based analysis of diversity term and accuracy of the resulting loss function approximation. We also presented results on many UCI data sets for two frequently encountered situations using ensembles of experts. These results demonstrate the utility of the proposed decomposition to ensembles used in these real-world problems.

Future work should design supervised learning and ensemble selection algorithms utilizing the proposed GAD theorem. Such algorithms might extend existing work on training diverse ensembles of neural networks using negative correlation learning~\cite{liu1999ensemble} and conditional maximum entropy models~\cite{audhkhasi2012creating}. The GAD loss function approximation is especially attractive because the diversity term does not require labeled data for computation. This opens the possibility of developing semi-supervised learning algorithms which use large amounts of unlabeled data.

Another interesting research direction is understanding the impact of diversity introduced at various stages of a conventional supervised learning algorithm on the final ensemble performance. The individual experts can be trained on different data sets and can use different feature sets. It would be useful to understand the most beneficial ways in which diversity can be introduced in the ensemble. We would also
like to study the impact of diversity in sequential classifiers such as automatic speech recognition
(ASR) systems. Our recent work~\cite{audhkhasi13theoretical} develops a GAD-like framework for 
theoretically analyzing the impact of diversity on fusion performance of state-of-the-art ASR systems.
Finally, characterizing diversity in an ensembles of human experts presents a tougher challenge because it is difficult to quantify the underlying loss function. However, many real-world problems involving crowd-sourcing~\cite{snow08, callison09, ambati10, marge10, denkowski10, sorokin08, heer10, audhkhasi2011accurate, audhkasi2011reliability} and understanding human behavior involve annotation by multiple human experts~\cite{raykar10, yan10, welinder10, dawid79, wollmer08, audhkhasi2013globally, audhkhasi2011emotion, audhkhasi2010data}. Extending the GAD theorem to such cases will contribute significantly to these domains.

\bibliographystyle{IEEEbib}
\bibliography{gad_tpami_refs,kartik_refs}

\end{document}